\documentclass{article}

\usepackage[final]{neurips_2022}

\usepackage[utf8]{inputenc} 
\usepackage[T1]{fontenc}    
\usepackage{hyperref}       
\usepackage{url}            
\usepackage{booktabs}       
\usepackage{amsfonts}       
\usepackage{nicefrac}       
\usepackage{microtype}      
\usepackage{xcolor}         

\usepackage{mathtools}   
\usepackage{algorithm}
\usepackage{algorithmic}  

\usepackage{amsmath}  
\usepackage{amsthm}
\usepackage{amssymb}
\usepackage{mathtools}
\usepackage{verbatim}
\usepackage{multirow}
\usepackage{subcaption}

\newtheorem{theorem}{Theorem}
\newtheorem{proposition}{Proposition}

\newtheorem{remark}{Remark}

\DeclareMathOperator*{\argmin}{arg\,min}
\DeclareMathOperator*{\minimize}{\mathrm{minimize}}
\newcommand{\sqnorm}[1]{\left\| #1 \right\|^2}
\newcommand{\Exp}[1]{\mathbb{E}\!\left[ #1 \right]}
\newcommand{\oma}{\omega_{\mathrm{av}}}

\usepackage{pifont}
\definecolor{mygreen1}{rgb}{0,0.8,0}
\newcommand{\cmark}{\textcolor{mygreen1}{\ding{51}}}%
\newcommand{\xmark}{\textcolor{red}{\ding{55}}}%

\usepackage{xspace}
\usepackage{scalefnt}

\definecolor{mydarkgreen}{rgb}{0,0.42,0}
\definecolor{mydarkred}{rgb}{0.75,0,0}
\definecolor{mygreen2}{RGB}{0,120,20}
\newcommand{\algname}[1]{{\sf\color{mydarkred}\scalefont{0.90}{#1}}\xspace}

\newcommand{\eqdef}{\coloneqq}

\title{\bf EF-BV: A Unified Theory of Error Feedback  and Variance Reduction Mechanisms for Biased  and Unbiased Compression in Distributed Optimization}

\author{%
  Laurent Condat\thanks{Corresponding author. Contact: see https://lcondat.github.io/} \\
  KAUST  
   \And
Kai Yi \\
KAUST
   \AND
   Peter Richtárik\\
   King Abdullah University of Science and Technology (KAUST)
  \\
Thuwal 23955-6900, Kingdom of Saudi Arabia   
}

\begin{document}
 
\maketitle

\begin{abstract}
In distributed or federated optimization and learning, communication between the different computing units is often the bottleneck and gradient compression is widely used to reduce the number of bits sent within each communication round of iterative methods. There are two classes of compression operators and separate algorithms making use of them. In the case of unbiased random compressors with bounded variance (e.g., rand-k), the \algname{DIANA} algorithm of \citet{mis19}, which implements a variance reduction technique for handling the variance introduced by compression, is the current state of the art. In the case of biased and contractive compressors (e.g., top-k), the \algname{EF21}  algorithm of \citet{ric21}, which instead implements an error-feedback mechanism, 
	is the current state of the art. These two classes of compression schemes and algorithms are distinct, with different analyses and proof techniques. In this paper, we unify them into a single framework and propose a new algorithm, recovering \algname{DIANA} and \algname{EF21}  as particular cases. Our general approach works with a new, larger class of compressors, which has two parameters, the bias and the variance, and 
	includes unbiased and biased compressors as particular cases. 
	This allows us to inherit the best of the two worlds: like \algname{EF21}  and unlike \algname{DIANA}, biased compressors, like top-k, whose good performance in practice is recognized, can be used. And like \algname{DIANA} and unlike \algname{EF21}, independent randomness at the compressors allows to mitigate the effects of compression, with the convergence rate improving when the number of parallel workers is large. This is the first time that an algorithm with all these features is proposed. We prove its linear convergence under certain conditions. 
	Our approach takes a step towards better understanding of two so-far distinct worlds of communication-efficient distributed learning.
	\end{abstract}

	\section{Introduction}

	In the big data era, the explosion in size and complexity of the data
	arises in parallel to a shift towards distributed computations~\citep{ver21}, 
	as modern hardware increasingly relies on the power of uniting many parallel units into one system. For distributed optimization and learning tasks, specific issues arise, such as
	decentralized data storage. In the modern paradigm of \emph{federated learning}~\citep{ja2016,mcm17,kai19,li20}, a potentially huge number of devices, with their owners' data stored on each of them, are involved in the collaborative process of 
	training a global machine learning model. The goal is to exploit the wealth of useful information lying in the \emph{heterogeneous} data stored across the network of such devices.  
	But users are increasingly sensitive to privacy concerns and prefer their data to never leave their devices. Thus, the devices have to \emph{communicate} the right amount of information back and forth with a distant server, for this distributed learning process to work. 
	Communication, which can be costly and slow, is the main bottleneck in this framework. So, it is of primary importance to devise novel algorithmic strategies, which are efficient in terms of computation and communication complexities. A natural and widely used idea is 
	 to make use of (lossy) {\em compression}, to reduce the size of the communicated messages \citep{ali17,wen17,wan18,GDCI,alb20,bas20,dut20,sat20,xu21}. 
	
	 In this paper, we propose a stochastic gradient descent (\algname{SGD})-type method for distributed optimization, which uses possibly \emph{biased} and randomized compression operators. Our algorithm is variance-reduced \citep{han19,gor202,gow20a}; that is, it converges to the exact solution, with fixed stepsizes, without any restrictive assumption on the functions to minimize.

\noindent\textbf{Problem.}\ \  We consider the convex optimization problem
	\begin{equation}  
	\minimize_{x\in\mathbb{R}^d}\;\underbrace{\frac{1}{n}\sum_{i=1}^n f_i(x)}_{f(x)} + R(x),\label{eqpro1}
	\end{equation}
	where $d\geq 1$ is the model dimension; 
	$R:\mathbb{R}^d\rightarrow \mathbb{R}\cup\{+\infty\}$ is a  proper, closed, convex function \citep{bau17}, whose proximity operator 
$\mathrm{prox}_{\gamma R} : x \mapsto \argmin_{y\in \mathbb{R}^d} \big( \gamma R(y)+\frac{1}{2}\|x-y\|^2 \big)$ 
is easy to compute, for any $\gamma>0$~\citep{par14,con19,con22};
	$n\geq 1$ is the number of functions; each function $f_i:\mathbb{R}^d \rightarrow\mathbb{R}$ is convex and $L_i$-smooth, for some $L_i>0$; that is, $f_i$ is differentiable on $\mathbb{R}^d$ and its gradient $\nabla f_i$ is $L_i$-Lipschitz continuous: for every $x\in\mathbb{R}^d$ and $x'\in\mathbb{R}^d$, 
	$\|\nabla f_i(x)-\nabla f_i(x')\|\leq L_i \|x-x'\|$.

	We set $L_{\max}\eqdef\max_i L_i$ and $\tilde{L}\eqdef\sqrt{\frac{1}{n}\sum_{i=1}^n L_i^2}$. 
	The average function 
	$f\eqdef \frac{1}{n}\sum_{i=1}^n f_i$
	is $L$-smooth, for some $L\leq \tilde{L} \leq L_{\max}$. 
	A minimizer of $f+R$ is supposed to exist.
 For any integer $m\geq 1$, we define the set $\mathcal{I}_m\eqdef \{1,\ldots,m\}$.\medskip
	
\noindent\textbf{Algorithms and Prior Work.}\ \ 	
Distributed proximal \algname{SGD} solves the problem \eqref{eqpro1} by iterating
$	x^{t+1} \eqdef \mathrm{prox}_{\gamma R} \big(x^t -  \frac{\gamma}{n} \sum_{i=1}^n g_i^t\big)$, 
where $\gamma$ is a stepsize and the vectors $g_i^t$ are possibly stochastic estimates of the gradients $\nabla f_i(x^t)$, which are cheap to compute or communicate.
 Compression is typically performed by the application of a possibly randomized operator $\mathcal{C}:\mathbb{R}^d\rightarrow \mathbb{R}^d$; that is, for any $x$, $\mathcal{C}(x)$ denotes a realization of a random variable, whose probability distribution depends on $x$. Compressors have  the property that it is much easier/faster to transfer $\mathcal{C}(x)$ than the original message $x$. This can be achieved in several ways, for instance by sparsifying the input vector~\citep{ali18}, or by quantizing its entries~\citep{ali17,Cnat,gan19,may21,sah21}, or via a combination of these and other approaches~\citep{Cnat,alb20, bez20}. There are two classes of compression operators often studied in the literature: 1) unbiased compression operators, satisfying a variance bound proportional to the squared norm of the input vector, and 2) biased compression operators, whose square distortion is contractive with respect to the squared norm of the input vector; we present these two classes in Sections \ref{secun} and \ref{secbia}, respectively. \medskip

\noindent\textbf{Prior work: \algname{DIANA} with unbiased compressors.}\ \ 
An important contribution to the field in the recent years is the variance-reduced \algname{SGD}-type method called \algname{DIANA}~\citep{mis19}, which uses unbiased compressors; it is shown in Fig.~\ref{fig1}.  \algname{DIANA} was analyzed and extended in several ways, including bidirectional compression and acceleration,  see, e.g., the work of \citet{hor22,mis20,con22m,phi20,li2020,gor20}, and \citet{gor202,kha20} for general theories about \algname{SGD}-type methods, including variants using unbiased compression of (stochastic) gradients.\medskip

\noindent\textbf{Prior work: Error feedback with biased contractive compressors.}\ \ 
Our understanding of distributed optimization using biased compressors is more limited. The key complication comes from the fact that their naive use within methods like  gradient descent can lead to divergence, as widely observed in practice, see also Example~1 of  \citet{bez20}. 
\emph{Error feedback} (\algname{EF}), also called error compensation, techniques were proposed to fix this issue and obtain convergence, initially as heuristics \citep{sei14}. Theoretical advances have been made in the recent years in the analysis of \algname{EF}, see the discussions and references in \citet{ric21} and \citet{chu22}. But the question of whether it is possible to obtain a linearly convergent \algname{EF} method in the general heterogeneous data setting, relying on biased compressors only, was still an open problem; until last year, 2021, when \citet{ric21} re-engineered the classical \algname{EF} mechanism and came up with a new algorithm, called \algname{EF21}. 
 It was then extended in several ways, including by considering server-side compression, and the support of a regularizer $R$ in \eqref{eqpro1}, by \citet{fat21}. \algname{EF21}  is shown in Fig.~\ref{fig1}.\medskip

\noindent\textbf{Motivation and challenge.}\ \ 
While \algname{EF21} resolved an important theoretical problem  in the field of distributed optimization with contractive compression, there are still several open questions. In particular,  \algname{DIANA} with independent random compressors has a $\frac{1}{n}$ factor in its iteration complexity; that is, it converges faster when the number $n$ of workers is larger. 
\algname{EF21} does not have this property: its convergence rate does not depend on $n$. Also, the convergence analysis and proof techniques for the two algorithms are different: the linear convergence analysis of  \algname{DIANA} relies on $\|x^t-x^\star\|^2$ and $\|h_i^t-\nabla f_i(x^\star)\|^2$ tending to zero, where $x^t$ is the estimate of the solution $x^\star$ at iteration $t$ and $h_i^t$ is the control variate maintained at node $i$, whereas the analysis of \algname{EF21} relies on $(f+R)(x^t)-(f+R)(x^\star)$ and $\|h_i^t-\nabla f_i(x^t)\|^2$ tending to zero, and under different assumptions. This work aims at filling this gap. 
That is, we want to address the following open problem:

\begin{table*}[t]
\caption{Desirable properties of a distributed compressed gradient descent algorithm converging to an exact solution of \eqref{eqpro1} and whether they are satisfied by the state-of-the-art algorithms \algname{DIANA} and \algname{EF21} and their currently-known analysis, and the proposed algorithm \algname{EF-BV}.}  
\label{tab1}
\centering
\begin{tabular}{ccccc}
& \algname{DIANA}&\algname{EF21}&\algname{EF-BV}\\
\hline
handles unbiased compressors in $\mathbb{U}(\omega)$ for any $\omega\geq 0$&\cmark&\cmark${}^1$&\cmark\\
\hline
handles biased contractive compressors in $\mathbb{B}(\alpha)$ for any $\alpha\in (0,1]$&\xmark&\cmark&\cmark\\
\hline
handles  compressors in $\mathbb{C}(\eta,\omega)$ for any $\eta\in [0,1)$, $\omega\geq 0$&\xmark&\cmark${}^1$&\cmark\\
\hline
recovers  \algname{DIANA} and \algname{EF21} as particular cases&\xmark&\xmark&\cmark\\
\hline
the convergence rate improves when $n$ is large&\cmark&\xmark&\cmark\\
\hline
\end{tabular}

{\small ${}^1$: with pre-scaling with $\lambda<1$, so that $\mathcal{C}'= \lambda\mathcal{C}\in\mathbb{B}(\alpha)$ is used instead of $\mathcal{C}$}
\end{table*}

{\itshape
Is it possible to design an algorithm, which combines the advantages of  \algname{DIANA} and \algname{EF21}? That is, such that:
\begin{enumerate}
	\item[a.]
It deals with unbiased compressors, biased contractive compressors, and possibly even more.
	\item[b.]
It recovers  \algname{DIANA} and \algname{EF21} as particular cases.
\item[c.]
Its convergence rate improves with $n$ large.  
	\end{enumerate}%
}\medskip
	
\noindent\textbf{Contributions.}\ \ We answer positively this question and propose a new algorithm, which we name \algname{EF-BV}, for \emph{Error Feedback with Bias-Variance decomposition}, which for the first time satisfies the three aforementioned properties. This is illustrated in Tab.~\ref{tab1}. More precisely, our contributions are:
\begin{enumerate}
	\item We propose a new, larger class of compressors,
	which includes unbiased and biased contractive compressors as particular cases, and has two parameters, the \textbf{bias} $\eta$ and the \textbf{variance} $\omega$. A third parameter $\oma$ describes  the resulting variance from the parallel compressors after aggregation, and is key to getting faster convergence with large $n$, by allowing larger stepsizes than in \algname{EF21} in our framework.
	\item We propose a new algorithm, named \algname{EF-BV}, which exploits the properties of the compressors in the new class using two scaling parameters $\lambda$ and $\nu$. For particular values of $\lambda$ and $\nu$, \algname{EF21} and  \algname{DIANA} are recovered as particular cases. But by setting the values of $\lambda$ and $\nu$ optimally with respect to $\eta$, $\omega$, $\oma$ in \algname{EF-BV}, faster convergence can be obtained.
	\item We prove linear convergence of \algname{EF-BV} under a Kurdyka--{\L}ojasiewicz condition of $f+R$, which is weaker than strong convexity of $f+R$. 	Even for \algname{EF21} and  \algname{DIANA}, this is new.
	
	\item We provide new insights on \algname{EF21} and  \algname{DIANA}; for instance, we prove linear convergence of  \algname{DIANA} with biased compressors.

	\end{enumerate}%

\section{Compressors and their properties}

\subsection{Unbiased compressors}\label{secun}

For every $\omega\geq 0$, we introduce the set $\mathbb{U}(\omega)$ of unbiased compressors, which are randomized operators of the form $\mathcal{C}:\mathbb{R}^d\rightarrow \mathbb{R}^d$, satisfying
\begin{equation}
\mathbb{E}[\mathcal{C}(x)]=x\quad\mbox{and}\quad \mathbb{E}[\|\mathcal{C}(x)-x\|^2]\leq \omega\|x\|^2,\quad \forall x\in\mathbb{R}^d,
\label{eq4}
\end{equation}
where $\Exp{\cdot}$ denotes the expectation. 
The smaller $\omega$, the better, and $\omega=0$ if and only if $\mathcal{C}=\mathrm{Id}$, the identity operator, which does not compress.  We can remark that if $\mathcal{C}\in\mathbb{U}(\omega)$ is deterministic, then $\mathcal{C}=\mathrm{Id}$. So, unbiased compressors are random ones. A classical unbiased compressor is \texttt{rand-}$k$, for some $k\in \mathcal{I}_d$, which keeps $k$ elements chosen uniformly at random, multiplied by $\frac{d}{k}$, and sets the other elements to 0.  It is easy to see that \texttt{rand-}$k$ belongs to $\mathbb{U}(\omega)$ with $\omega=\frac{d}{k}-1$ \citep{bez20}.

\subsection{Biased contractive compressors}\label{secbia}

For every $\alpha\in (0,1]$, we introduce the set $\mathbb{B}(\alpha)$ of biased contractive compressors, which are possibly randomized operators of the form $\mathcal{C}:\mathbb{R}^d\rightarrow \mathbb{R}^d$, satisfying
\begin{equation}
 \mathbb{E}[\|\mathcal{C}(x)-x\|^2]\leq (1-\alpha)\|x\|^2,\quad \forall x\in\mathbb{R}^d.\label{eq5}
\end{equation}
We use the term `contractive' to reflect the fact that the squared norm in the left hand side of \eqref{eq5} is smaller, in expectation, than the one  in the right hand side, since $1-\alpha <1$. This is not the case in \eqref{eq4}, where $\omega$ can be arbitrarily large. 
The larger $\alpha$, the better, and $\alpha=1$ if and only if $\mathcal{C}=\mathrm{Id}$. 
Biased compressors need not be random: a classical biased and deterministic compressor is \texttt{top-}$k$, for some $k\in\mathcal{I}_d$, which keeps the $k$ elements with largest absolute values unchanged and sets the other elements to 0. It is easy to see that \texttt{top-}$k$ belongs to $\mathbb{B}(\alpha)$ with $\alpha=\frac{k}{d}$ \citep{bez20}.

\subsection{New general class of compressors}\label{sec23}

We refer to \citet{bez20}, Table 1 in \citet{saf21}, \citet{zha21}, \citet{sze22}, for examples of compressors in $\mathbb{U}(\omega)$ or $\mathbb{B}(\alpha)$, and to \citet{xu20} for a system-oriented survey.

In this work, we introduce a new, more general class of compressors, ruled by 2 parameters, to allow for a finer characterization of their properties. Indeed, with any compressor $\mathcal{C}$, we can do a {\bf bias-variance decomposition} of the compression error: 
for every $x\in\mathbb{R}^d$,
\begin{equation}
 \mathbb{E}\big[\|\mathcal{C}(x)-x\|^2\big] = {\underbrace{\big\| \mathbb{E}[\mathcal{C}(x)]-x\big\|}_{\text{bias}}}^2 + \underbrace{\mathbb{E}\Big[\big\|\mathcal{C}(x)-\mathbb{E}[\mathcal{C}(x)]\big\|^2\Big]}_{\text{variance}}.\label{eqbiva}
 \end{equation}
Therefore, to better characterize the properties of compressors, we propose to parameterize these two parts, instead of only their sum: for every  $\eta \in [0,1)$ and $\omega\geq 0$, 
we introduce the new class $\mathbb{C}(\eta,\omega)$ of possibly random and biased operators, which are randomized operators of the form $\mathcal{C}:\mathbb{R}^d\rightarrow \mathbb{R}^d$, satisfying, for every $x\in\mathbb{R}^d$, the two properties:
\begin{align*}
\mathrm{(i)}\quad &\big\| \mathbb{E}[\mathcal{C}(x)]-x\big\|\leq \eta \|x\|,\\
\mathrm{(ii)}\quad& \mathbb{E}\Big[\big\|\mathcal{C}(x)-\mathbb{E}[\mathcal{C}(x)]\big\|^2\Big]\leq \omega\|x\|^2.
\end{align*}
Thus, $\eta$ and $\omega$ control the relative bias and variance of the compressor, respectively. 
Note that $\omega$ can be arbitrarily large, but the compressors will be scaled in order to control the compression error, as we discuss  in Sect.~\eqref{secsca}. 
 On the other hand, we must have $\eta<1$, since otherwise, no scaling can keep the compressor's discrepancy under control.\medskip

We have the following properties:
\begin{enumerate}
\item  $\mathbb{C}(\eta,0)$ is the class of deterministic compressors in $\mathbb{B}(\alpha)$, with $1-\alpha=\eta^2$.

\item  $\mathbb{C}(0,\omega)=\mathbb{U}(\omega)$, for every $\omega\geq 0$. In words, if its bias $\eta$ is zero,  the compressor is unbiased with relative variance $\omega$.

\item  Because of the bias-variance decomposition \eqref{eqbiva},  if $\mathcal{C}\in\mathbb{C}(\eta,\omega)$ with $\eta^2+\omega < 1$, then $\mathcal{C}\in\mathbb{B}(\alpha)$ with 
\begin{equation}
1-\alpha = \eta^2+\omega.\label{eqalpha}
\end{equation}

\item  Conversely, if $\mathcal{C}\in\mathbb{B}(\alpha)$, one easily sees from \eqref{eqbiva} that there exist $\eta \leq \sqrt{1-\alpha}$ and $\omega \leq 1-\alpha$ such that $\mathcal{C}\in\mathbb{C}(\eta,\omega)$.
\end{enumerate}

Thus, the new class $\mathbb{C}(\eta,\omega)$  generalizes the two previously known classes $\mathbb{U}(\omega)$ and $\mathbb{B}(\alpha)$. Actually, for compressors in $\mathbb{U}(\omega)$ and $\mathbb{B}(\alpha)$, we can just use  \algname{DIANA} and  \algname{EF21}, and our proposed algorithm \algname{EF-BV} will stand out when the compressors are neither in $\mathbb{U}(\omega)$ nor in $\mathbb{B}(\alpha)$; that is why the strictly larger class $\mathbb{C}(\eta,\omega)$ is needed for our purpose.

We present new compressors in the class $\mathbb{C}(\eta,\omega)$ in Appendix~\ref{secappa}.

\subsection{Average variance of several compressors}\label{secavv}

Given $n$ compressors $\mathcal{C}_i$, $i\in \mathcal{I}_n$, we are interested in how they behave in average. Indeed distributed algorithms consist, at every iteration, in compressing vectors in parallel, and then averaging them. Thus, we 
 introduce the \textbf{average relative variance} 
$\oma\geq 0$ of the compressors,  
such that, for every $x_i\in\mathbb{R}^d$, $i\in\mathcal{I}_n$, 
\begin{equation}
\Exp{ \sqnorm{ \frac{1}{n}\sum_{i=1}^n \big(\mathcal{C}_i(x_i)-\Exp{\mathcal{C}_i(x_i)}\big)} } \leq \frac{\oma}{n} \sum_{i=1}^n \sqnorm{x_i }.
\label{eqbo}
\end{equation}
When every $\mathcal{C}_i$ is in $\mathbb{C}(\eta,\omega)$, for some $\eta \in [0,1)$ and $\omega\geq 0$, then 
 $\oma \leq \omega$; but $\oma$ can be much smaller than  $\omega$, and we will exploit this property in \algname{EF-BV}. We can also remark that $
\frac{1}{n} \sum_{i=1}^n \mathcal{C}^i \in \mathbb{C}(\eta,\oma)$.

An important property is the following: if the $\mathcal{C}_i$ are mutually independent, since the variance of a sum of random variables is the sum of their variances, then
\begin{equation*}
\oma=\frac{\omega}{n}.
\end{equation*}
There are other cases where the compressors are dependent but $\oma$ is much smaller than $\omega$. Notably, the following setting can be used to model partial participation of $m$ among $n$ workers at every iteration of a distributed algorithm, for some $m\in \mathcal{I}_n$, with the $\mathcal{C}_i$ defined jointly as follows: for every  $i\in \mathcal{I}_n$ and $x_i\in\mathbb{R}^d$, 
\begin{equation*}
\mathcal{C}_i(x_i) =\begin{cases} \;\frac{n}{m} x_i & \text{ if }\;  i\in\Omega \\ 
\;0 & \text{ otherwise} \end{cases},
\end{equation*}
where $\Omega$ is a subset of $\mathcal{I}_n$ of size $m$ chosen uniformly at random. 
 This is sometimes called $m$-nice sampling \citep{ric16,gow20}. Then  every $\mathcal{C}_i$ belongs to $\mathbb{U}(\omega)$, with $\omega=\frac{n-m}{m}$, and, as shown for instance in   \citet{qia19} and Proposition~1 in \citet{con22m},  \eqref{eqbo} is satisfied with
\begin{equation*}
\oma=\frac{n-m}{m(n-1)}=\frac{\omega}{n-1}\quad\mbox{($=0$ if $n=m=1$)}. 
\end{equation*}

\subsection{Scaling compressors}\label{secsca}

A compressor $\mathcal{C}\in\mathbb{C}(\eta,\omega)$ does not necessarily belong to $\mathbb{B}(\alpha)$ for any $\alpha \in (0,1]$, since $\omega$ can be arbitrarily large. 
Fortunately, the compression error can be kept under control 
by \emph{scaling} the compressor; that is, using  $\lambda \mathcal{C}$ instead of $\mathcal{C}$, for some scaling parameter $\lambda\leq 1$. We have:
\begin{proposition}
\label{prop3}
Let $\mathcal{C}\in\mathbb{C}(\eta,\omega)$, for some  $\eta \in [0,1)$ and $\omega\geq 0$, and $\lambda\in (0,1]$. Then
$\lambda\mathcal{C}\in\mathbb{C}(\eta',\omega')$ with $\omega'= \lambda^2 \omega$ and $\eta '=\lambda\eta+1-\lambda\in(0,1]$.\end{proposition}

\begin{proof}Let $x\in\mathbb{R}^d$. Then 
$\mathbb{E}\Big[\big\|\lambda\mathcal{C}(x)-\mathbb{E}[\lambda\mathcal{C}(x)]\big\|^2\Big]  = \lambda^2\mathbb{E}\Big[\big\|\mathcal{C}(x)-\mathbb{E}[\mathcal{C}(x)]\big\|^2\Big] \leq  \lambda^2\omega\|x\|^2,
$ 
 and 
$\big\| \mathbb{E}[\lambda\mathcal{C}(x)]-x\big\| 
 \leq  \lambda
\big\| \mathbb{E}[\mathcal{C}(x)]-x\big\|+(1-\lambda) \|x\| \leq  (\lambda\eta+1-\lambda)\|x\| .
$
\end{proof}

So, scaling deteriorates the bias, with $\eta'\geq \eta$, but linearly, whereas it reduces the variance $\omega$ quadratically. This is key, since 
the total error factor $(\eta')^2+\omega'$ can be made smaller than 1 by choosing $\lambda$ sufficiently small:
\begin{proposition}
\label{propsmall}
Let $\mathcal{C}\in\mathbb{C}(\eta,\omega)$, for some  $\eta \in [0,1)$ and $\omega\geq 0$. There exists $\lambda\in(0,1]$ such that 
$\lambda\mathcal{C}\in \mathbb{B}(\alpha)$, for some $\alpha = 1- (1-\lambda+\lambda\eta)^2-{\lambda}^2\omega \in (0,1]$, and the best such $\lambda$, maximizing $\alpha$, is
\begin{equation*}
\lambda^\star=\min\left(\frac{1-\eta}{(1-\eta)^2+\omega},1\right).
\end{equation*}
\end{proposition}

\begin{proof}
We define the polynomial $P:\lambda\mapsto(1-\lambda+\lambda\eta)^2+\lambda^2\omega$. 
After Proposition~\ref{prop3} and the discussion in Sect.~\ref{sec23}, we have to find $\lambda\in(0,1]$ such that $P(\lambda)<1$. Then 
$\lambda\mathcal{C}\in \mathbb{B}(\alpha)$, with $1-\alpha=P(\lambda)$. Since $P$ is a strictly convex quadratic function on $[0,1]$ with value $1$ and negative derivative $\eta-1$ at $\lambda=0$,  its minimum value on $[0,1]$ is smaller than 1 and is attained at $\lambda^\star$, which either satisfies the first-order condition $0 = P'(\lambda) = -2(1-\eta) +2\lambda\big((1-\eta)^2 + \omega\big)$, or, if this value is larger than 1, is equal to 1.
\end{proof}

In particular, if $\eta=0$, Proposition~\ref{propsmall} recovers Lemma 8 of \citet{ric21}, according to which, for  $\mathcal{C}\in\mathbb{U}(\omega)$, $\lambda^\star \mathcal{C}\in \mathbb{B}(\frac{1}{\omega+1})$, with $\lambda^\star=\frac{1}{\omega+1}$. 
For instance, the scaled \texttt{rand-}$k$ compressor, 
which keeps $k$ elements chosen uniformly at random unchanged and sets the other elements to 0, corresponds to scaling the unbiased \texttt{rand-}$k$ compressor, seen in Sect.~\ref{secun}, by $\lambda=\frac{k}{d}$.

We can remark that scaling is used to mitigate the randomness of a compressor, but cannot be used to reduce its bias: if $\omega=0$, $\lambda^\star=1$.

Our new algorithm \algname{EF-BV} will have two scaling parameters: $\lambda$, to mitigate the compression error in the control variates used for variance reduction, just like above, and $\nu$, to mitigate the error in the stochastic gradient estimate, in a similar way but with $\omega$ replaced by $\oma$, since we have seen in Sect.~\ref{secavv} that $\oma$ characterizes the randomness after averaging 
the outputs of several compressors.

\section{Proposed algorithm \algname{EF-BV}}

We propose the algorithm \algname{EF-BV}, shown in Fig.~\ref{fig1}. It makes use of compressors $\mathcal{C}_i^t \in \mathbb{C}(\eta,\omega)$, for some  $\eta \in [0,1)$ and $\omega\geq 0$, and we introduce $\oma\leq \omega$ such that \eqref{eqbo} is satisfied. That is, for any $x\in\mathbb{R}^d$, the $\mathcal{C}_i^t(x)$, for $i\in\mathcal{I}_n$ and $t\geq 0$, are distinct random variables; their laws might be the same or not, but they all lie in the class $\mathbb{C}(\eta,\omega)$. Also, $\mathcal{C}_i^t(x)$ and $\mathcal{C}_{i'}^{t'}(x')$, for $t\neq t'$, are independent. 

The compressors have the property that if their input is the zero vector, the compression error is zero, so we want to compress vectors that are close to zero, or at least converge to zero, to make the method variance-reduced. That is why each worker maintains a control variate $h_i^t$, converging, like $\nabla f_i(x^t)$, to $\nabla f_i(x^\star)$, for 
some solution $x^\star$. This way, the difference vectors $\nabla f_i(x^t)-h_i^t$ converge to zero, and these are the vectors that are going to be compressed. 
Thus,  \algname{EF-BV} takes the form of Distributed proximal \algname{SGD}, with $$g_i^t = h_i^t + \nu  \mathcal{C}_i^t\big(\nabla f_i(x^t)-h_i^t\big),$$ where the scaling parameter $\nu$ will be used to make the compression error, averaged over $i$, small; that is, to make $g^{t+1}=\frac{1}{n}\sum_{i=1}^n g_i^t $ close to $\nabla f(x^t)$. 
In parallel, the control variates are updated similarly as $$h_i^{t+1}= h_i^t + \lambda \mathcal{C}_i^t\big(\nabla f_i(x^t)-h_i^t\big),$$ where the scaling parameter $\lambda$ is used to make the compression error small, individually for each $i$; that is, to make $h_i^{t+1}$ close to $\nabla f_i(x^t)$.

\begin{figure*}[t]
\begin{minipage}{.31\textwidth}
	\begin{algorithm}[H]
	\scalefont{0.9}
		\caption{\algname{EF-BV} \\ proposed method}
		\begin{algorithmic}
			\STATE
			\noindent \textbf{Input:} $x^0, h_1^0, \dots, h_n^0 \in \mathbb{R}^d$,  
			$h^0= \frac{1}{n}\sum_{i=1}^n h_i^0$, $\gamma>0$, \\
			 {\color{blue}$\lambda\in(0,1]$}, {\color{red}$\nu\in(0,1]$}
			\FOR{$t=0, 1, \ldots$}
			\FOR{$i=1, 2, \ldots, n$ in parallel}
			\STATE $d_i^t \coloneqq \mathcal{C}_i^t\big(\nabla f_i(x^t)-h_i^t\big)$
			\STATE $h_i^{t+1} \coloneqq h_i^t + {\color{blue}\lambda} d_i^t$
			\STATE send $d_i^t$ to master
			\ENDFOR
			\STATE at master:
			\STATE $d^t \coloneqq \frac{1}{n}\sum_{i=1}^n d_i^t$
			\STATE $h^{t+1} \coloneqq h^t + {\color{blue}\lambda} d^t$
			\STATE $g^{t+1} \coloneqq h^t + {\color{red}\nu} d^t$ 
			\STATE $x^{t+1}\!\coloneqq\!\mathrm{prox}_{\gamma R}(x^t - \gamma g^{t+1})$%
			\STATE broadcast $x^{t+1}$ to all workers
			\ENDFOR
		\end{algorithmic}
	\end{algorithm}
	\end{minipage}
	\ \ \ \ \ \ \begin{minipage}{.31\textwidth}
	\begin{algorithm}[H]
	\scalefont{0.9}
		\caption{\algname{EF21} \\ \citep{ric21}}
		\begin{algorithmic}
			\STATE
			\noindent \textbf{Input:} $x^0, h_1^0, \dots, h_n^0 \in \mathbb{R}^d$,  
			$h^0= \frac{1}{n}\sum_{i=1}^n h_i^0$,  $\gamma>0$, \\
			\phantom{XXX}
			\FOR{$t=0, 1, \ldots$}
			\FOR{$i=1, 2, \ldots, n$ in parallel}
			\STATE $d_i^t \coloneqq \mathcal{C}_i^t\big(\nabla f_i(x^t)-h_i^t\big)$
			\STATE $h_i^{t+1} \coloneqq h_i^t +  d_i^t$
			\STATE send $d_i^t$ to master
			\ENDFOR
			\STATE at master:
			\STATE $d^t \coloneqq \frac{1}{n}\sum_{i=1}^n d_i^t$
			\STATE $h^{t+1} \coloneqq h^t +  d^t$
			\STATE $g^{t+1} \coloneqq h^t +  d^t$ 
			\STATE $x^{t+1}\!\coloneqq\!\mathrm{prox}_{\gamma R}(x^t - \gamma g^{t+1})$%
			\STATE broadcast $x^{t+1}$ to all workers
			\ENDFOR
		\end{algorithmic}
	\end{algorithm}
	\end{minipage}
	\ \ \ \ \ \ \begin{minipage}{.31\textwidth}
	\begin{algorithm}[H]
	\scalefont{0.9}
		\caption{\algname{DIANA} \\ \citep{mis19}}
		\begin{algorithmic}
			\STATE
			\noindent \textbf{Input:}  $x^0, h_1^0, \dots, h_n^0 \in \mathbb{R}^d$,  $h^0= \frac{1}{n}\sum_{i=1}^n h_i^0$, 
			  $\gamma>0$, \\
			{\color{blue}$\lambda\in(0,1]$}
			\FOR{$t=0, 1, \ldots$}
			\FOR{$i=1, 2, \ldots, n$ in parallel}
			\STATE $d_i^t \coloneqq \mathcal{C}_i^t\big(\nabla f_i(x^t)-h_i^t\big)$
			\STATE $h_i^{t+1} \coloneqq h_i^t + {\color{blue}\lambda} d_i^t$
			\STATE send $d_i^t$ to master
			\ENDFOR
			\STATE at master:
			\STATE $d^t \coloneqq \frac{1}{n}\sum_{i=1}^n d_i^t$
			\STATE $h^{t+1} \coloneqq h^t + {\color{blue}\lambda} d^t$ 			
			\STATE $g^{t+1} \coloneqq h^t +  d^t$ 
			\STATE $x^{t+1}\!\coloneqq\!\mathrm{prox}_{\gamma R}(x^t - \gamma g^{t+1})$%
			\STATE broadcast $x^{t+1}$ to all workers
			\ENDFOR
		\end{algorithmic}
	\end{algorithm}
	\end{minipage}
	\caption{\label{fig1}In the three algorithms, $g^{t+1}$ is an estimate of $\nabla f(x^t)$, the $h_i^t$ are control variates converging to 
	$\nabla f_i(x^\star)$, and their average $h^t= \frac{1}{n}\sum_{i=1}^n h_i^t$ is maintained and updated by the master. 
	\algname{EF21} is a particular case of \algname{EF-BV}, when $\nu=\lambda=1$ and the compressors are in $\mathbb{B}(\alpha)$; 	then $g^{t+1}$ is simply equal to $h^{t+1}$ for every $t\geq 0$. 
	\algname{DIANA} is a particular case of \algname{EF-BV}, when $\nu=1$ and  the compressors are in $\mathbb{U}(\omega)$; then $g^{t}$ is an unbiased estimate of $\nabla f(x^t)$.}
	\end{figure*}

	\subsection{\algname{EF21} as a particular case of \algname{EF-BV}}\label{sec31}
	
	There are two ways to recover \algname{EF21} as a particular case of \algname{EF-BV}:
	
\begin{enumerate}
\item If the compressors $\mathcal{C}_i^t$ are in $\mathbb{B}(\alpha)$, for some $\alpha\in (0,1]$, there is no need for scaling the compressors, and we can use \algname{EF-BV} with $\lambda=\nu=1$. Then the variable $h^t$ in \algname{EF-BV} becomes redundant with the gradient estimate $g^t$ and we can only keep the latter, which yields  \algname{EF21}, as shown in Fig.~\ref{fig1}.
	
\item  If the scaled compressors $\lambda\mathcal{C}_i^t$ are in $\mathbb{B}(\alpha)$, for some $\alpha\in (0,1]$ and $\lambda\in (0,1)$
	(see Proposition~\ref{propsmall}), one  can simply 	use these scaled compressors in \algname{EF21}. This is equivalent to using  \algname{EF-BV} with the original compressors $\mathcal{C}_i^t$, the scaling with $\lambda$ taking place inside the algorithm. But we must have $\nu=\lambda$ for this equivalence to hold.
\end{enumerate}	
	
	Therefore, we consider thereafter that \algname{EF21} corresponds to the particular case of  \algname{EF-BV} with $\nu=\lambda \in (0,1]$ and $\lambda\mathcal{C}_i^t \in \mathbb{B}(\alpha)$, for some $\alpha\in (0,1]$, and is not only the original algorithm shown in Fig.~\ref{fig1},  which has no scaling parameter (but scaling might have been applied beforehand to make the compressors in  $\mathbb{B}(\alpha)$).

	\subsection{\algname{DIANA} as a particular case of \algname{EF-BV}}\label{sec32}
	
	\algname{EF-BV} with $\nu=1$  yields exactly  \algname{DIANA}, as shown in Fig.~\ref{fig1}.  \algname{DIANA} was only studied with unbiased compressors $\mathcal{C}_i^t\in\mathbb{U}(\omega)$, for some $\omega\geq 0$. In that case, $\Exp{g^{t+1}}=\nabla f(x^t)$, so that $g^{t+1}$ is an unbiased stochastic gradient estimate; this is not the case in \algname{EF21} and \algname{EF-BV}, in general. Also, 
	$\lambda=\frac{1}{1+\omega}$ is the usual choice in  \algname{DIANA}, which is consistent with Proposition~\ref{propsmall}.

\section{Linear convergence results}\label{sec5}

We will prove linear convergence of \algname{EF-BV} under conditions weaker than strong convexity of $f+R$.

When $R=0$, we will consider the  Polyak--{\L}ojasiewicz (P{\L}) condition on $f$: $f$ is said to satisfy the  P{\L}  condition with constant $\mu>0$ if, 
 for every $x\in\mathbb{R}^d$,
$\|\nabla f(x)\|^2\geq 2\mu\big(f(x)-f^\star\big)$,
where $f^\star = f(x^\star)$, for any minimizer $x^\star$ of $f$. This holds if, for instance, $f$ is $\mu$-strongly convex; that is, $f-\frac{\mu}{2}\|\cdot\|^2$ is convex.  In the general case, we will consider the  Kurdyka--{\L}ojasiewicz (K{\L}) condition with exponent $1/2$ \citep{att09,kar16} on $f+R$: $f+R$ is said to satisfy the  K{\L}  condition with constant $\mu>0$ if,  for every $x\in\mathbb{R}^d$ and $u\in  \partial R(x)$, 
\begin{equation}
\|\nabla f(x)+u\|^2\geq 2\mu\big(f(x)+R(x)-f^\star-R^\star\big),\label{eqKL}
\end{equation} 
where $f^\star = f(x^\star)$ and $R^\star= R(x^\star)$, for any minimizer $x^\star$ of $f+R$. This holds if, for instance, $R=0$ and $f$ satisfies the P{\L}  condition with constant $\mu$, so that the K{\L}  condition generalizes the P{\L}  condition to the general case $R\neq 0$. The K{\L}  condition also holds if $f+R$ is $\mu$-strongly convex \citep{kar16}, for which it is sufficient that $f$ is $\mu$-strongly convex, or $R$ is $\mu$-strongly convex.

In the rest of this section, we assume that $\mathcal{C}_i^t \in \mathbb{C}(\eta,\omega)$, for some  $\eta \in [0,1)$ and $\omega\geq 0$, and we introduce $\oma\leq \omega$ such that \eqref{eqbo} is satisfied. 
According to the discussion in Sect.~\ref{secsca} (see also Remark~\ref{rem1} below), we define the optimal values for the scaling parameters $\lambda$ and $\nu$:
\begin{align*}
\lambda^\star\eqdef\min\left(\frac{1-\eta}{(1-\eta)^2+\omega},1\right), \qquad \nu^\star\eqdef\min\left(\frac{1-\eta}{(1-\eta)^2+\oma},1\right).
\end{align*}
Given $\lambda \in (0,1]$ and $\nu \in (0,1]$, 
we define for convenience
$r \eqdef (1-\lambda+\lambda\eta)^2+{\lambda}^2\omega$, $r_{\mathrm{av}} \eqdef (1-\nu+\nu\eta)^2+{\nu}^2\oma$, 
as well as
$s^\star \eqdef \sqrt{\frac{1+r}{2r}}-1$ and $\theta^\star \eqdef s^\star(1+s^\star)\frac{r}{r_{\mathrm{av}}}$.

Note that if $r<1$, according to Proposition~\ref{prop3} and \eqref{eqalpha}, $\lambda \mathcal{C}_i^t \in \mathbb{B}(\alpha)$, with $\alpha=1-r$.

Our linear convergence results for \algname{EF-BV} are the following: 

\begin{theorem}\label{theo1}Suppose that $R=0$ and $f$ satisfies the P{\L}  condition with some constant  $\mu>0$. 
In \algname{EF-BV}, suppose that $\nu \in (0,1]$, $\lambda \in (0,1]$ is such that $r<1$, and 
\begin{equation}
0<\gamma \leq \frac{1}{L+\tilde{L}\sqrt{\frac{r_{\mathrm{av}}}{r}}\frac{1}{s^\star}}.\label{equpb1}
\end{equation}
For every $t\geq 0$, define the Lyapunov function
$\displaystyle\Psi^t \eqdef f(x^t)-f^\star + \frac{\gamma}{2\theta^\star}  \frac{1}{n}\sum_{i=1}^n \sqnorm{\nabla f_i(x^t)-h_i^{t}}$, 
where $f^\star \eqdef f(x^\star)$, for any minimizer $x^\star$ of $f$. 
Then, for every $t\geq 0$,
\begin{align}
\Exp{\Psi^{t}} 
&\leq \left(\max\left(1-\gamma\mu, {\frac{r+1}{2}}\right) \right)^t\Psi^0.\label{eqsdgerg}
\end{align}
\end{theorem}

\begin{theorem}\label{theo2}
Suppose that $f+R$ satisfies the  the K{\L}  condition with some constant $\mu>0$. 
In \algname{EF-BV}, suppose that $\nu \in (0,1]$, $\lambda \in (0,1]$ is such that $r<1$, and
\begin{equation}
0<\gamma \leq \frac{1}{2L+\tilde{L}\sqrt{\frac{r_{\mathrm{av}}}{r}}\frac{1}{s^\star}}.\label{equpb2}
\end{equation}
 $\forall t\geq 0$, define the Lyapunov function
$\displaystyle\Psi^t \eqdef f(x^t)+R(x^t)-f^\star - R^\star  + \frac{\gamma}{2\theta^\star}  \frac{1}{n}\sum_{i=1}^n \sqnorm{\nabla f_i(x^t)-h_i^{t}}$, 
where $f^\star \eqdef f(x^\star)$ and $R^\star \eqdef R(x^\star)$, for any minimizer $x^\star$ of $f+R$. 
Then, for every $t\geq 0$,
\begin{align}
\Exp{\Psi^{t}}  &\leq \left(\max\left({\frac{1}{1+\frac{1}{2}\gamma\mu}},\frac{r+1}{2}\right)\right)^t\Psi^0.\label{eqsdgerg2}
\end{align}
\end{theorem}

\begin{remark}[choice of $\lambda$, $\nu$, $\gamma$ in \algname{EF-BV}]\label{rem1}
\normalfont In Theorems \ref{theo1} and \ref{theo2}, the rate is better if $r$ is small and $\gamma$ is large. So, we should take $\gamma$ equal to the upper bound in \eqref{equpb1} and \eqref{equpb2}, since there is no reason to choose it smaller. Also, this upper bound is large 
if $r$ and $r_{\mathrm{av}}$ are small. As discussed in Sect.~\ref{secsca}, $r$ and $r_{\mathrm{av}}$ are minimized with $\lambda=\lambda^\star$ and $\nu=\nu^\star$ (which implies that $ r_{\mathrm{av}}\leq r<1$), so this is the recommended choice. Also, with this choice of $\lambda$, $\nu$, $\gamma$, there is no parameter left to tune in the algorithm, which is a nice feature.
\end{remark}

\begin{remark}[low noise regime]\label{rem2}
When the compression error tends to zero, i.e.\ $\eta$ and $\omega$ tend to zero, and we use accordingly $\lambda \rightarrow 1$, $\nu \rightarrow 1$, such that $r_{\mathrm{av}}/r$ remains bounded, 
 then $\mathcal{C}_i^t\rightarrow\mathrm{Id}$, $r\rightarrow 0$, and $\frac{1}{s^\star}\rightarrow 0$. Hence, \algname{EF-BV} reverts to proximal gradient descent $x^{t+1} = \mathrm{prox}_{\gamma R} \big(x^t -\nabla f(x^t)\big)$. 
\end{remark}

\begin{remark}[high noise regime]\label{rem3}
When the compression error becomes large,  i.e.\ $\eta\rightarrow 1$ or $\omega\rightarrow +\infty$, then $r\rightarrow 1$ and $\frac{1}{s^\star}\sim \frac{4}{1-r}$. Hence, the asymptotic complexity of  \algname{EF-BV} to achieve $\epsilon$-accuracy, when $\gamma=\Theta\Big(\frac{1}{L+\tilde{L}\sqrt{\frac{r_{\mathrm{av}}}{r}}\frac{1}{s^\star}}\Big)$, is
\begin{equation}
\mathcal{O}\left(\left(\frac{L}{\mu}+\left(\frac{\tilde{L}}{\mu}\sqrt{\frac{r_{\mathrm{av}}}{r}}+1\right)\frac{1}{1-r}\right)
\log \frac{1}{\epsilon}\right).\label{eqasy1}
\end{equation}
\end{remark}

\subsection{Implications for \algname{EF21}}

Let us assume that $\nu=\lambda$, so that  \algname{EF-BV} reverts to  \algname{EF21}, as explained in Sect.~\ref{sec31}. 
Then, if we don't assume the prior knowledge of $\oma$, or equivalently if $\oma=\omega$, Theorem~\ref{theo1} with $r=r_{\mathrm{av}}$ recovers  the linear convergence result of  \algname{EF21} due to \citet{ric21}, up to slightly different constants.

However,  in these same conditions, Theorem~\ref{theo2} is new:  linear convergence of \algname{EF21} with  $R\neq 0$ was only shown in Theorem 13 of \citet{fat21}, under the assumption that there exists $\mu>0$, such that
for every $x\in\mathbb{R}^d$, 
$\frac{1}{\gamma^2} \sqnorm{x-\mathrm{prox}_{\gamma R}\big(x-\gamma \nabla f(x)
\big)} \geq 2\mu\big(f(x)+R(x)-f^\star-R^\star\big)$. 
This condition 
generalizes the P{\L}  condition, since it reverts to it when $R=0$, but it is different from the K{\L}  condition, and it is not clear when it is satisfied, in particular whether it is implied by strong convexity of $f+R$.

The asymptotic complexity  to achieve $\epsilon$-accuracy of \algname{EF21} with 
$\gamma=\Theta\big(\frac{1}{L+\tilde{L}/s^\star}\big)$ is
$\mathcal{O}\big(\frac{\tilde{L}}{\mu}\frac{1}{1-r}
\log \frac{1}{\epsilon}\big)$ 
(where we recall that $1-r=\alpha$, with the scaled compressors in $\mathbb{B}(\alpha)$).
Thus,  for a given problem and compressors, the improvement of  \algname{EF-BV} over  \algname{EF21} is the factor
$\sqrt{\frac{r_{\mathrm{av}}}{r}}$ in 
\eqref{eqasy1}, which can be small if $n$ is large.

Theorems \ref{theo1} and \ref{theo2} provide a new  insight about  \algname{EF21}: if we exploit the knowledge that $\mathcal{C}_i^t \in \mathbb{C}(\eta,\omega)$ and the corresponding constant $\oma$, and if $\oma<\omega$, then $r_{\mathrm{av}}<r$, so that, based on \eqref{equpb1} and \eqref{equpb2}, $\gamma$ can be chosen larger than with the default assumption  that $r_{\mathrm{av}}=r$. As a consequence, convergence will be faster. This illustrates the interest of our new finer parameterization of compressors with $\eta$, $\omega$, $\oma$. However, it is only half the battle to make use of the factor $\frac{r_{\mathrm{av}}}{r}$ in \algname{EF21}: the property $\oma<\omega$ is only really exploited if $\nu=\nu^\star$ in \algname{EF-BV} (since $r_{\mathrm{av}}$ is minimized this way). In other words, there is no reason to set $\nu=\lambda$ in \algname{EF-BV}, when a larger value of $\nu$ is allowed in Theorems \ref{theo1} and \ref{theo2} and yields faster convergence.

\subsection{Implications for  \algname{DIANA}}

Let us assume that $\nu=1$, so that  \algname{EF-BV} reverts to   \algname{DIANA}, as explained in Sect.~\ref{sec32}. This choice is allowed in Theorems \ref{theo1} and \ref{theo2}, so that they provide new convergence results for  \algname{DIANA}. Assuming that the compressors are unbiased, i.e.\ $\mathcal{C}_i^t\in \mathbb{U}(\omega)$ for some $\omega\geq 0$, we have the following result on   \algname{DIANA} \citep[Theorem 5 with $b=\sqrt{2}$]{con22m}:

\begin{proposition}\label{propdiana}
Suppose that $f$ is $\mu$-strongly convex, for some $\mu>0$, and that in  \algname{DIANA}, $\lambda=\frac{1}{1+\omega}$, 
$0<\gamma \leq \frac{1}{L_{\max}+L_{\max}(1+\sqrt{2})^2\oma}$. For every $t\geq 0$, define the Lyapunov function
$\Phi^t \eqdef \sqnorm{x^t-x^\star} 
+  (2+\sqrt{2})\gamma^2\oma (1+\omega)\frac{1}{n}\sum_{i=1}^n \sqnorm{\nabla f_i(x^\star)-h_i^t}$, 
where $x^\star$ is the minimizer of $f+R$, which exists and is unique. 
Then, for every $t\geq 0$,
$\Exp{\Phi^{t}} 
\leq \left(\max\left(1-\gamma\mu, \frac{\frac{1}{2}+\omega}{1+\omega}\right) \right)^t\Phi^0$.
\end{proposition}
Thus, noting that $r=\frac{\omega}{1+\omega}$, so that $\frac{r+1}{2}=\frac{\frac{1}{2}+\omega}{1+\omega}$, the rate is exactly the same as in Theorem~\ref{theo1}, but with a different Lyapunov function.  
Theorems \ref{theo1} and \ref{theo2} have the advantage over Proposition \ref{propdiana}, that linear convergence is guaranteed under the P{\L}  or K{\L}  assumptions, which are weaker than strong convexity of $f$. Also, the constants $L$ and $\tilde{L}$ appear instead of $L_{\max}$. This shows a better dependence with respect to the problem. However, noting that $r=\frac{\omega}{1+\omega}$, $r_{\mathrm{av}}=\oma$, $\frac{1}{s^\star}\sim 4\omega$, the factor $\sqrt{\frac{r_{\mathrm{av}}}{r}}\frac{1}{s^\star}$ scales like $\sqrt{\oma}\omega$, which is worse that $\oma$. This means that $\gamma$ can certainly be chosen larger in Proposition \ref{propdiana} than in Theorems \ref{theo1} and \ref{theo2}, leading to faster convergence.

However, Theorems \ref{theo1} and \ref{theo2} bring a major highlight: for the first time, they establish convergence of  \algname{DIANA}, which is  \algname{EF-BV} with $\nu=1$, with biased compressors. We state the results in Appendix~\ref{secappb}, by lack of space. 
In any case, 
with biased compressors, it is better to use \algname{EF-BV} than  \algname{DIANA}: there is no interest in choosing $\nu=1$ instead of $\nu=\nu^\star$, which minimizes $r_{\mathrm{av}}$ and allows for a larger $\gamma$, for faster convergence.

Finally, we can remark that for unbiased compressors with $\oma \ll 1$, for instance if $\oma \approx \frac{\omega}{n}$ with $n$ larger than $\omega$, then $\nu^\star=\frac{1}{1+\oma}\approx 1$. Thus, in this particular case,  \algname{EF-BV} with $\nu=\nu^\star$ and  \algname{DIANA} are essentially the same algorithm.  This is another sign that  \algname{EF-BV} with $\lambda=\lambda^\star$ and $\nu=\nu^\star$ is a generic and robust choice, since it recovers   \algname{EF21} and   \algname{DIANA} in settings where these algorithms shine.

\section{Sublinear convergence in the nonconvex case}

In this section, we consider the general nonconvex setting. In \eqref{eqpro1}, every function $f_i$ is supposed $L_i$-smooth, for some $L_i>0$. For simplicity, we suppose that $R=0$.
As previously, we set $\tilde{L}\eqdef\sqrt{\frac{1}{n}\sum_{i=1}^n L_i^2}$. The average function 
	$f\eqdef \frac{1}{n}\sum_{i=1}^n f_i$
	is $L$-smooth, for some $L\leq \tilde{L}$. We also suppose that $f$ is bounded from below; that is, $f^{\inf}\eqdef \inf_{x\in \mathbb{R}^d} f(x) > -\infty$.

 Given $\lambda \in (0,1]$ and $\nu \in (0,1]$, 
we define for convenience
$r \eqdef (1-\lambda+\lambda\eta)^2+{\lambda}^2\omega$, $r_{\mathrm{av}} \eqdef (1-\nu+\nu\eta)^2+{\nu}^2\oma$, 
as well as
$s \eqdef \frac{1}{\sqrt{r}}-1$ 
and $\theta \eqdef s(1+s)\frac{r}{r_{\mathrm{av}}}$. Our convergence result is the following:\smallskip

\begin{theorem}\label{thm:noncvx}
In \algname{EF-BV}, suppose that $\nu \in (0,1]$, $\lambda \in (0,1]$ is such that $r<1$, and 
       \begin{equation}0< \gamma \leq \frac{1}{L + \tilde{L}\sqrt{\frac{r_{\mathrm{av}}}{r}}\frac{1}{s}}.\end{equation}
     For every $t\geq 1$, let $\hat{x}^t$ be chosen from the iterates $x^0, x^1, \cdots, x^{t-1}$ uniformly at random. Then 
      \begin{equation}
         \mathbb{E}\left[\left\|\nabla f(\hat{x}^{t})\right\|^{2}\right] \leq \frac{2\big(f(x^{0})-f^{\inf}\big)}{\gamma t}+\frac{G^0}{\theta t},
         \end{equation}
      where $G^0\eqdef\frac{1}{n}\sum_{i=1}^n \sqnorm{\nabla f_i(x^0) - h_i^0}$.
   \end{theorem}

\section{Experiments}
We conducted comprehensive experiments to illustrate the efficiency of  \algname{EF-BV} compared to \algname{EF21} (we use biased compressors, so we don't include  \algname{DIANA} in the comparison). The settings and results are detailed in Appendix~\ref{appexp} and some results are shown in Fig.~\ref{fig:0007}; we can see the speedup obtained with \algname{EF-BV}, which exploits the randomness of the compressors.
 
\begin{figure}[!htbp]
	\centering
	\begin{subfigure}[b]{0.24\textwidth}
		\centering
		\includegraphics[width=\textwidth]{
		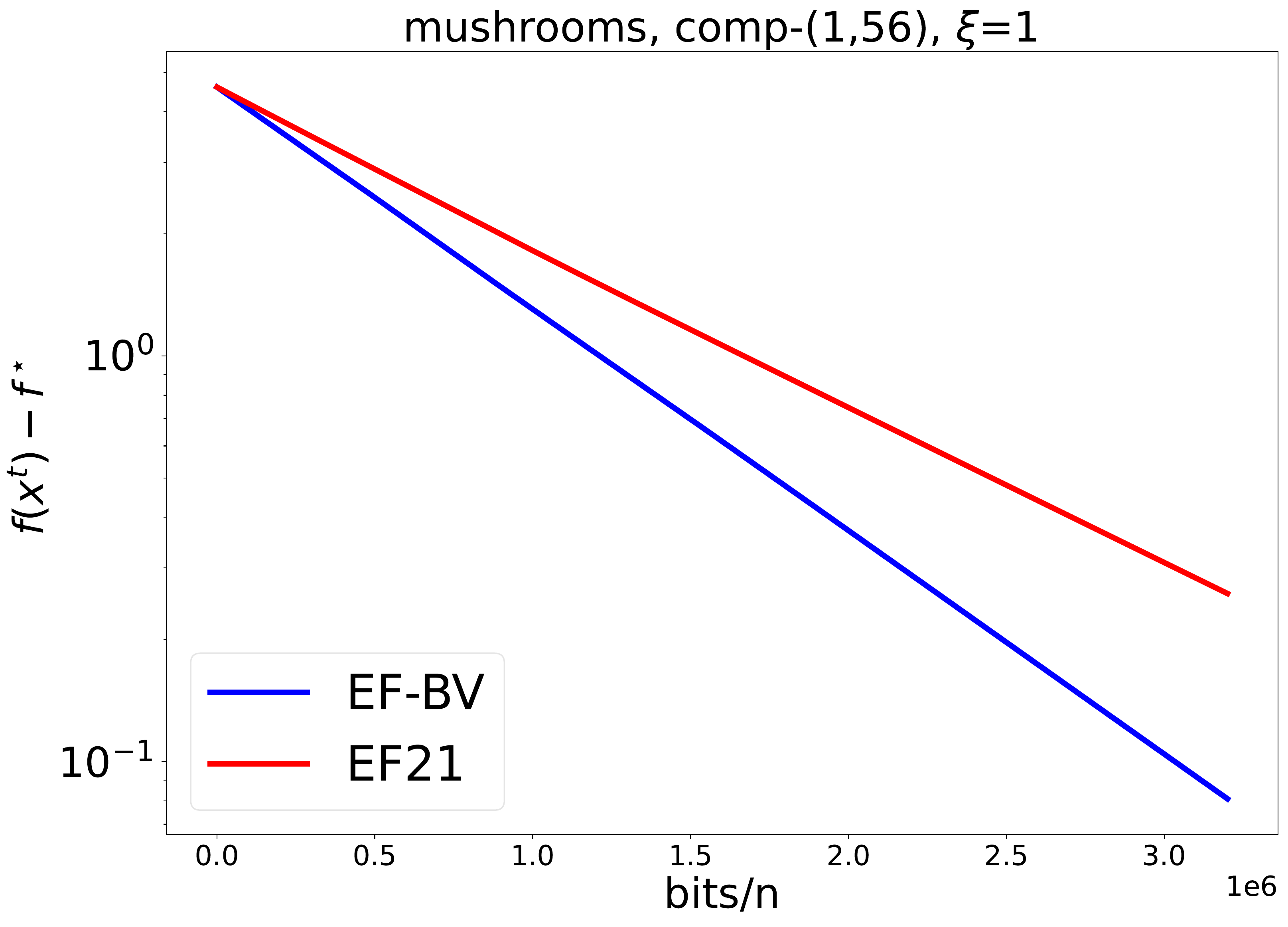}
	\end{subfigure}
	\hfill 
	\begin{subfigure}[b]{0.24\textwidth}
		\centering
		\includegraphics[width=\textwidth]{
		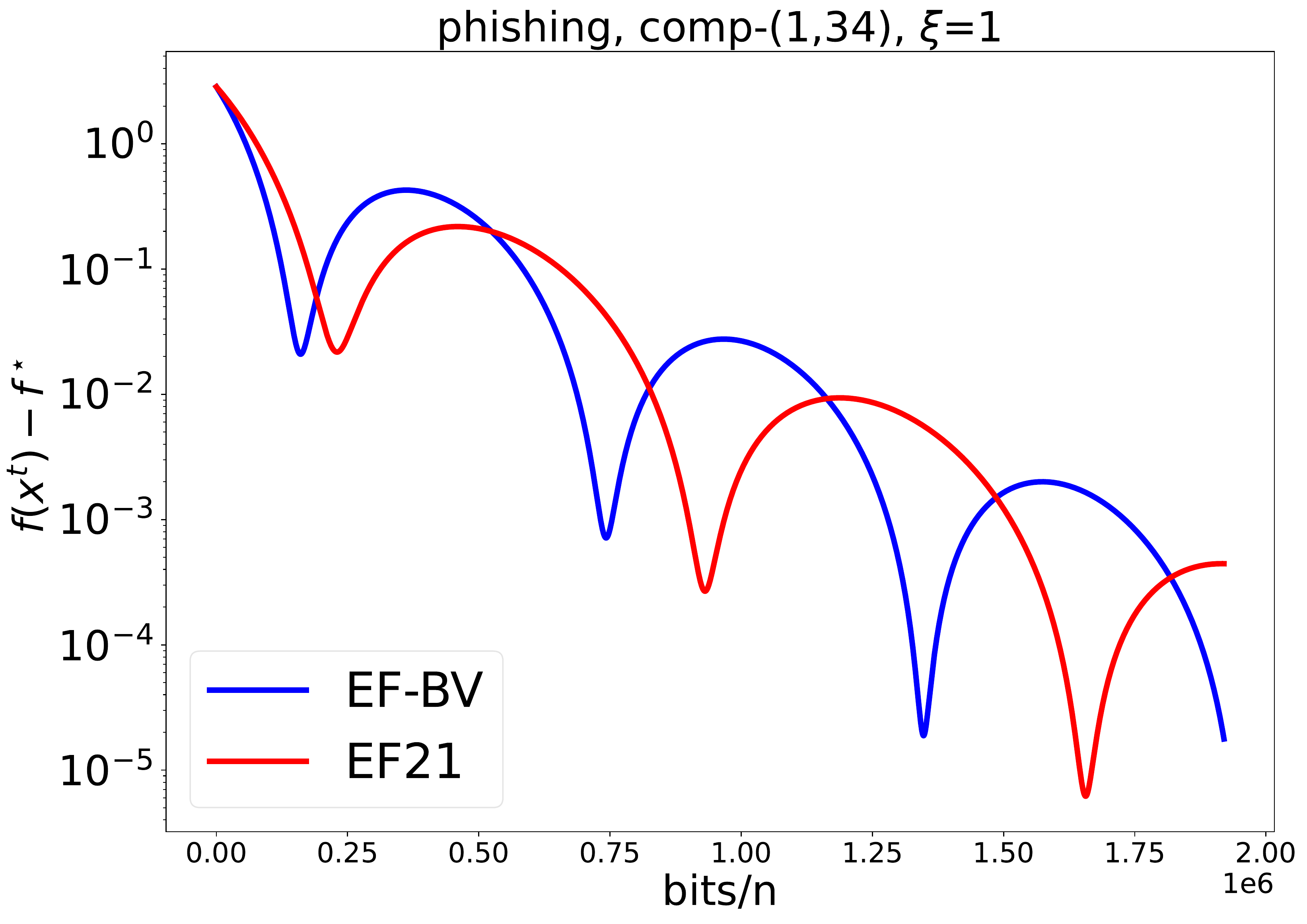}
	\end{subfigure}
	\hfill
	\begin{subfigure}[b]{0.24\textwidth}
		\centering
		\includegraphics[width=\textwidth]{
		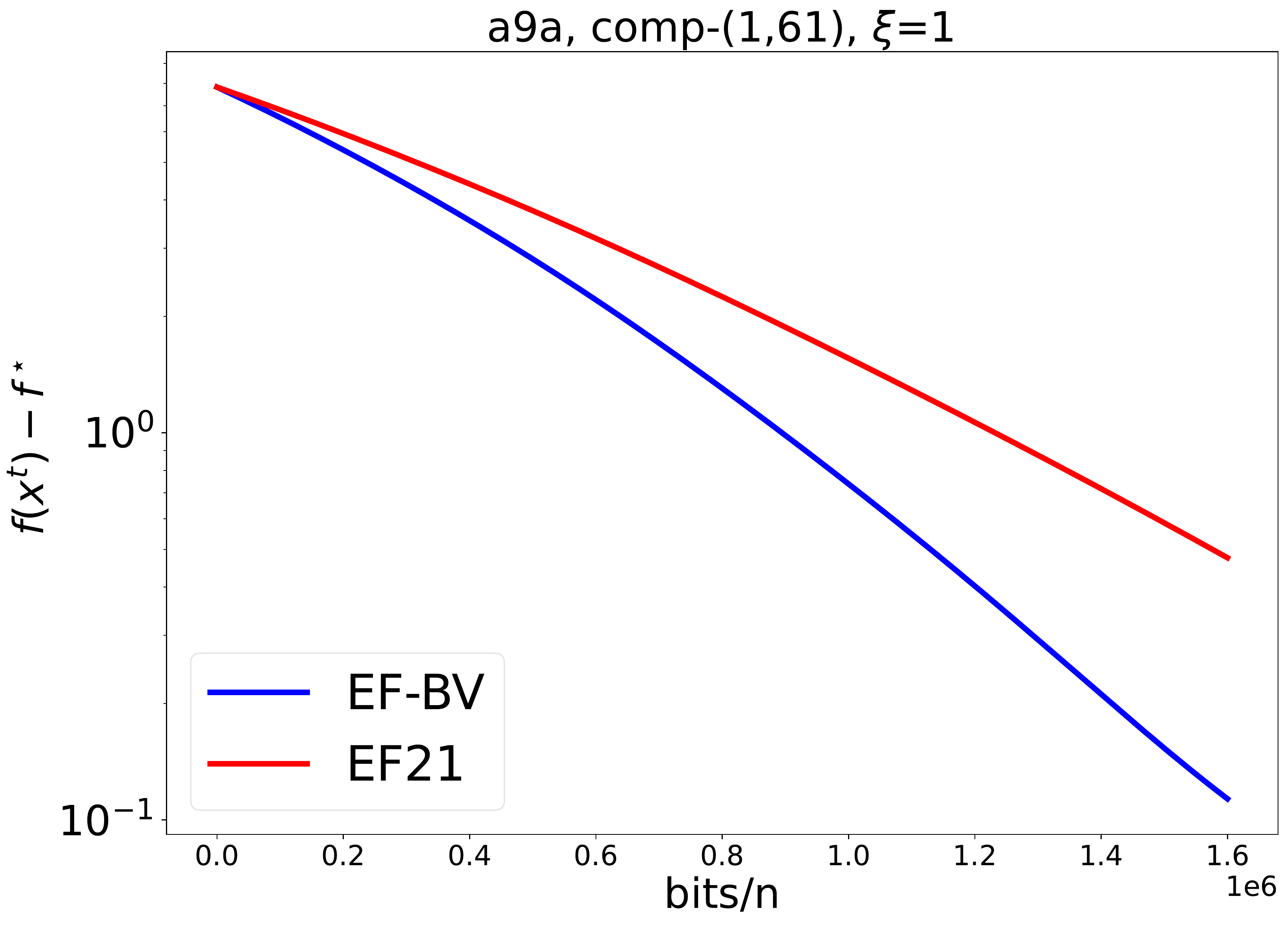}
	\end{subfigure}
	\begin{subfigure}[b]{0.24\textwidth}
		\centering
		\includegraphics[width=\textwidth]{
		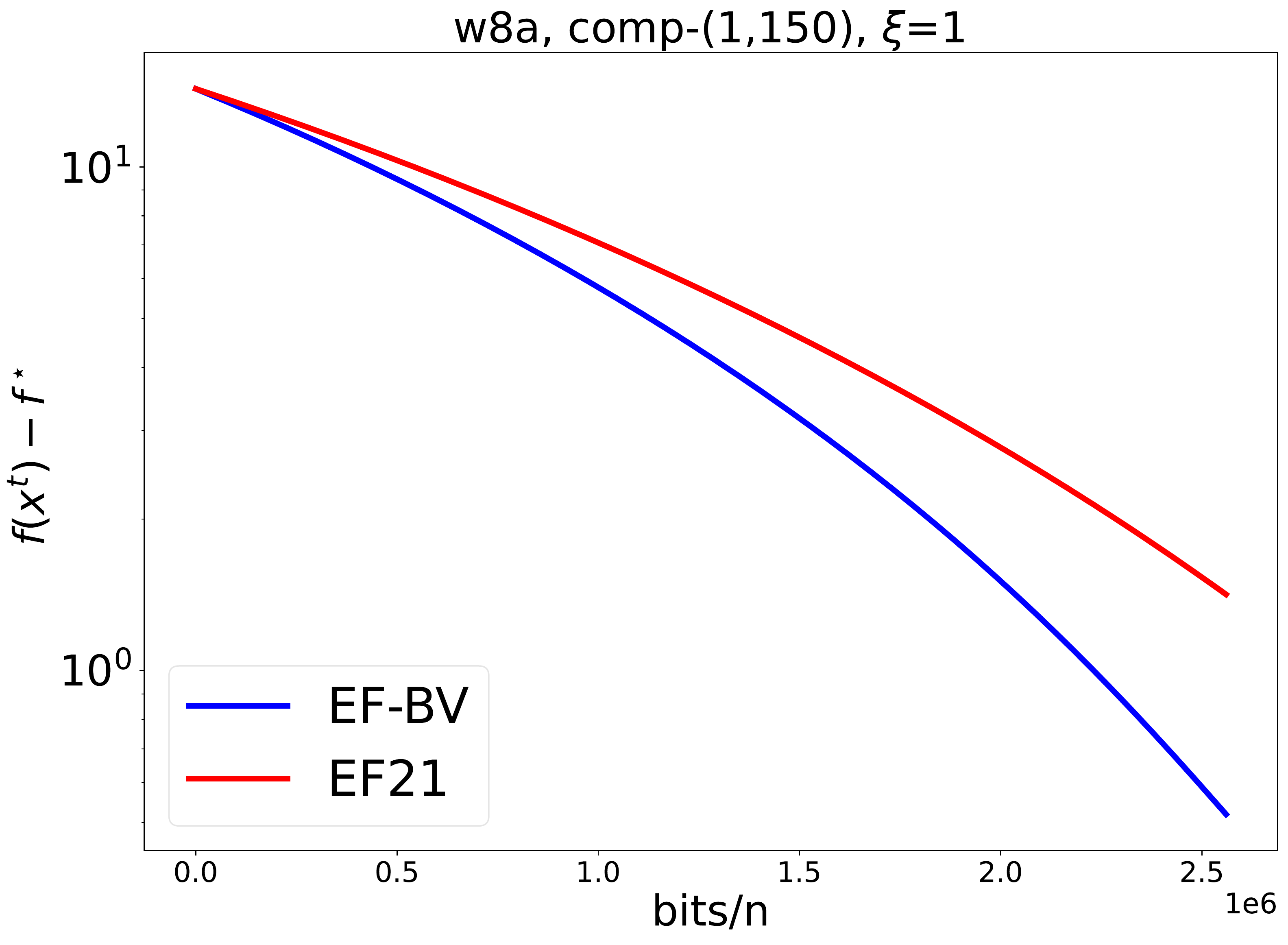}
	\end{subfigure}
	\hfill 
	\begin{subfigure}[b]{0.24\textwidth}
		\centering
		\includegraphics[width=\textwidth]{
		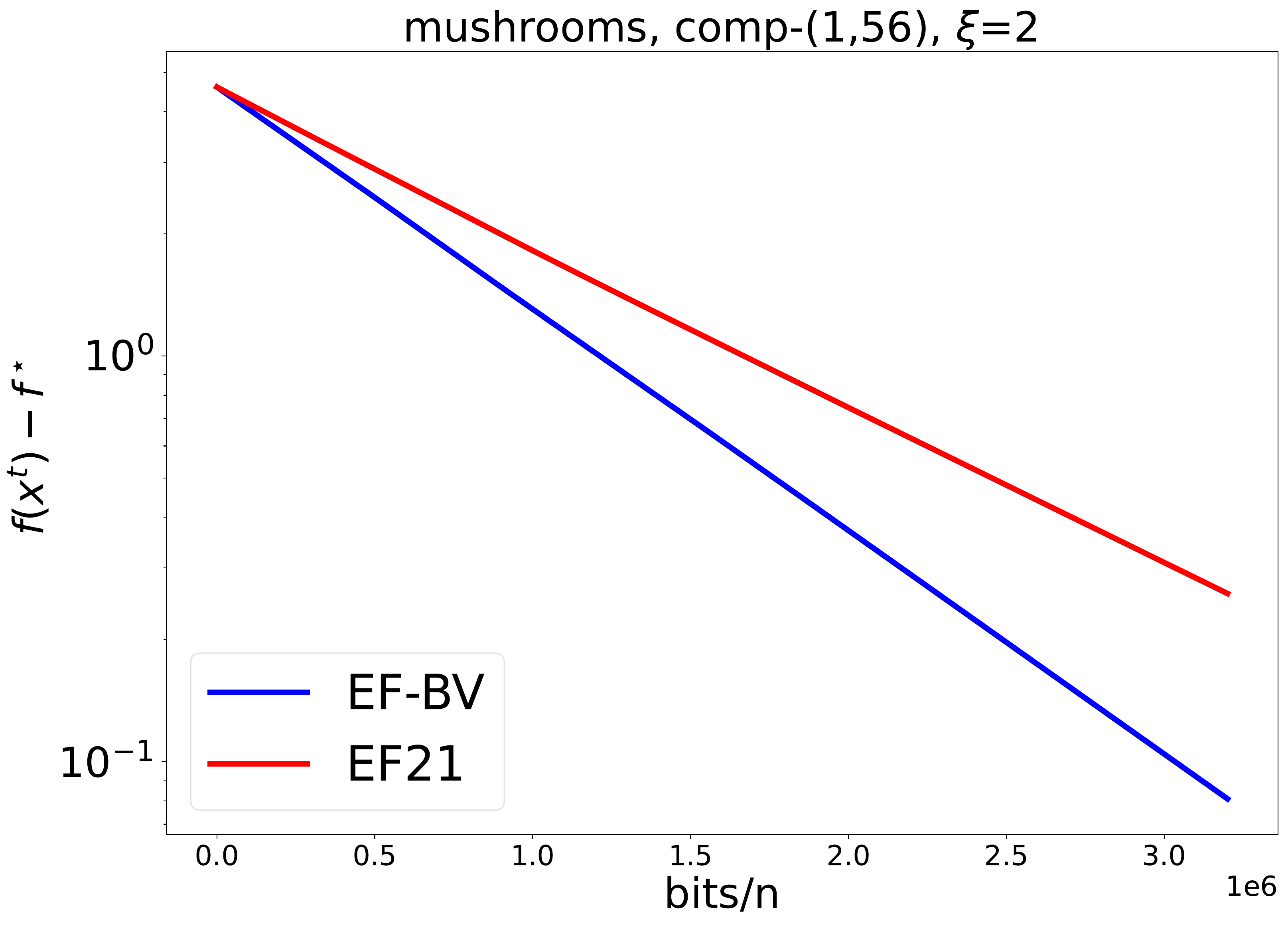}
	\end{subfigure}
	\hfill
	\begin{subfigure}[b]{0.24\textwidth}
		\centering
		\includegraphics[width=\textwidth]{
		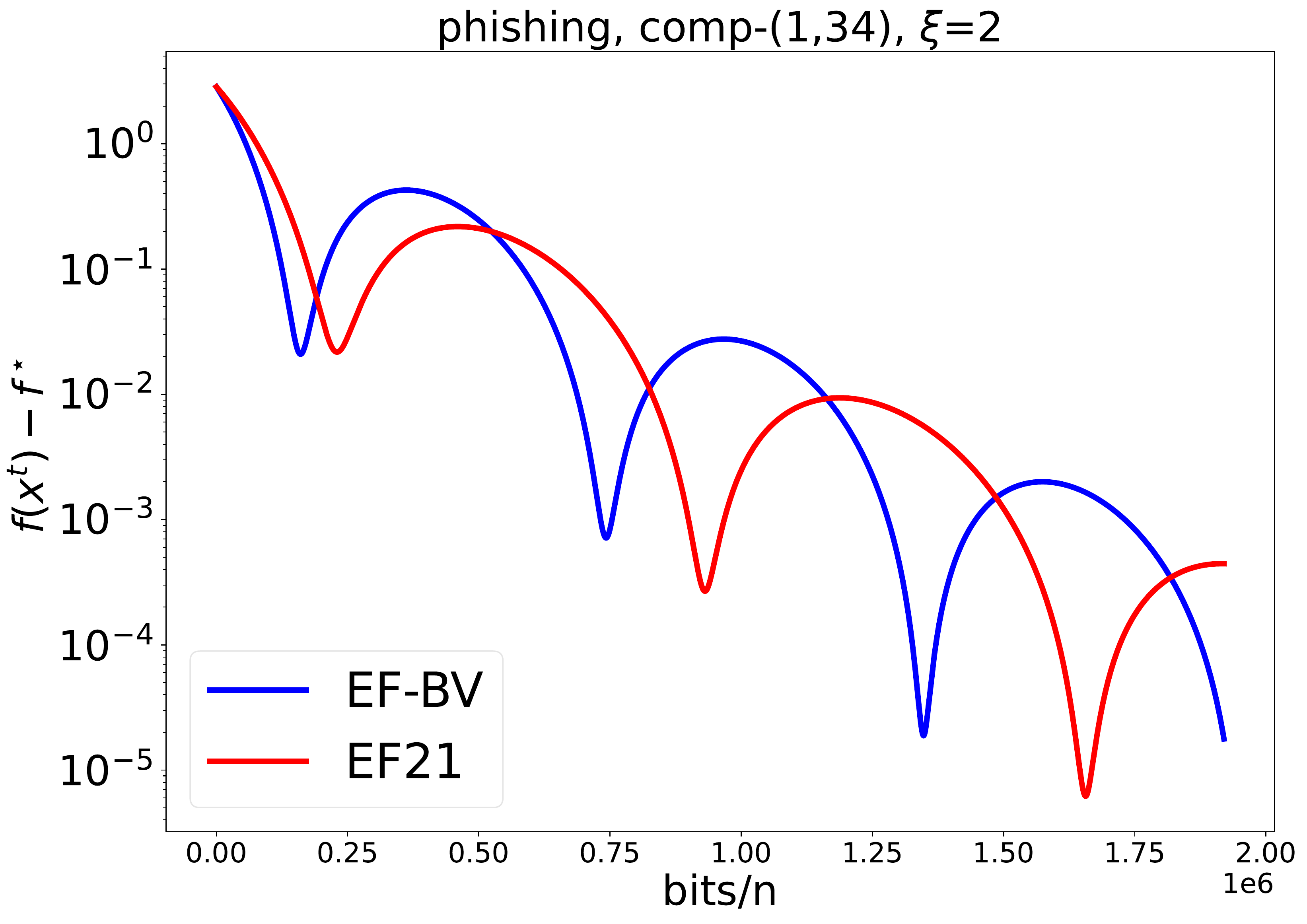}
	\end{subfigure}
	\hfill
	\begin{subfigure}[b]{0.24\textwidth}
		\centering
		\includegraphics[width=\textwidth]{
		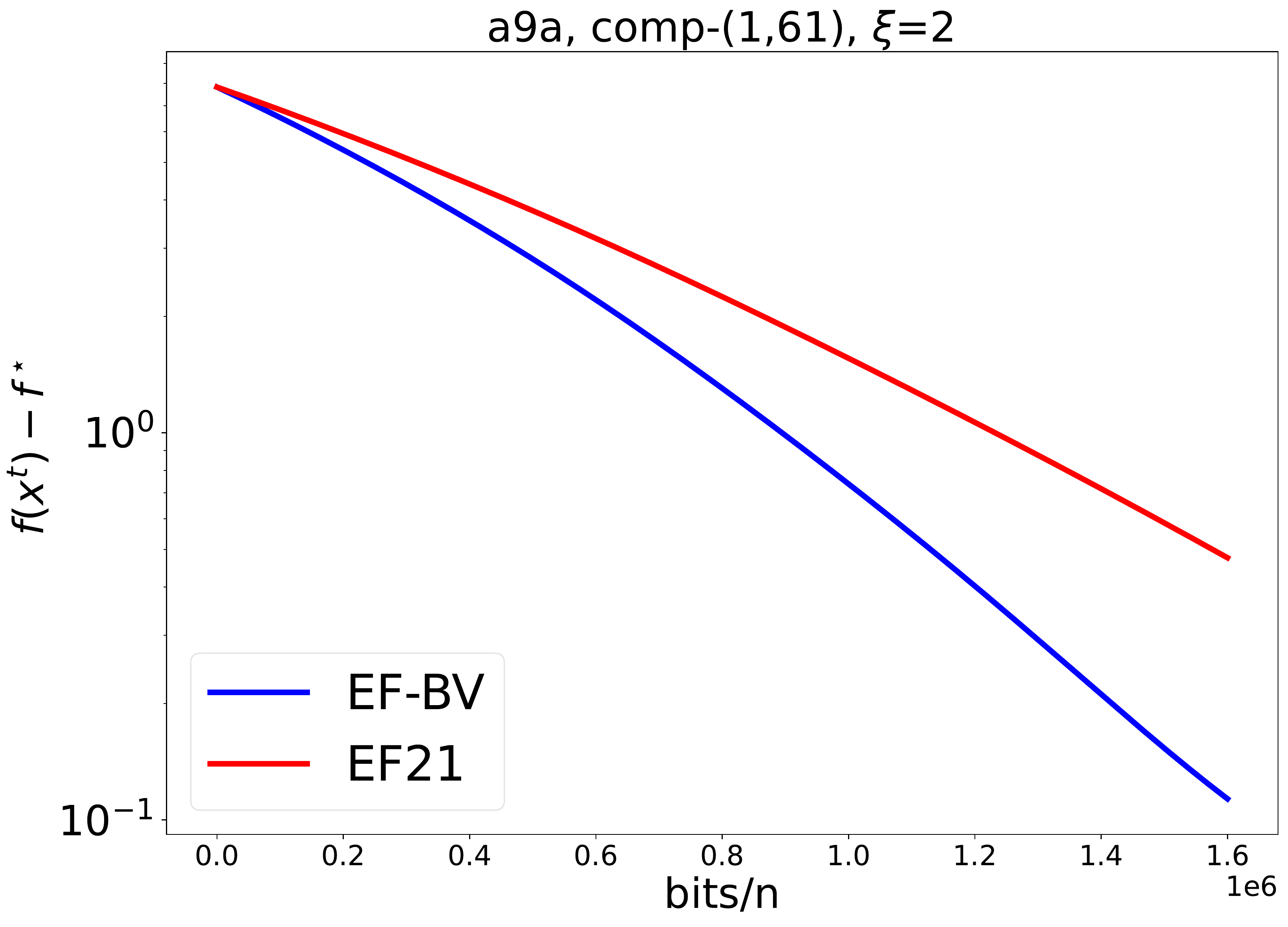}
	\end{subfigure}
	\begin{subfigure}[b]{0.24\textwidth}
		\centering
		\includegraphics[width=\textwidth]{
		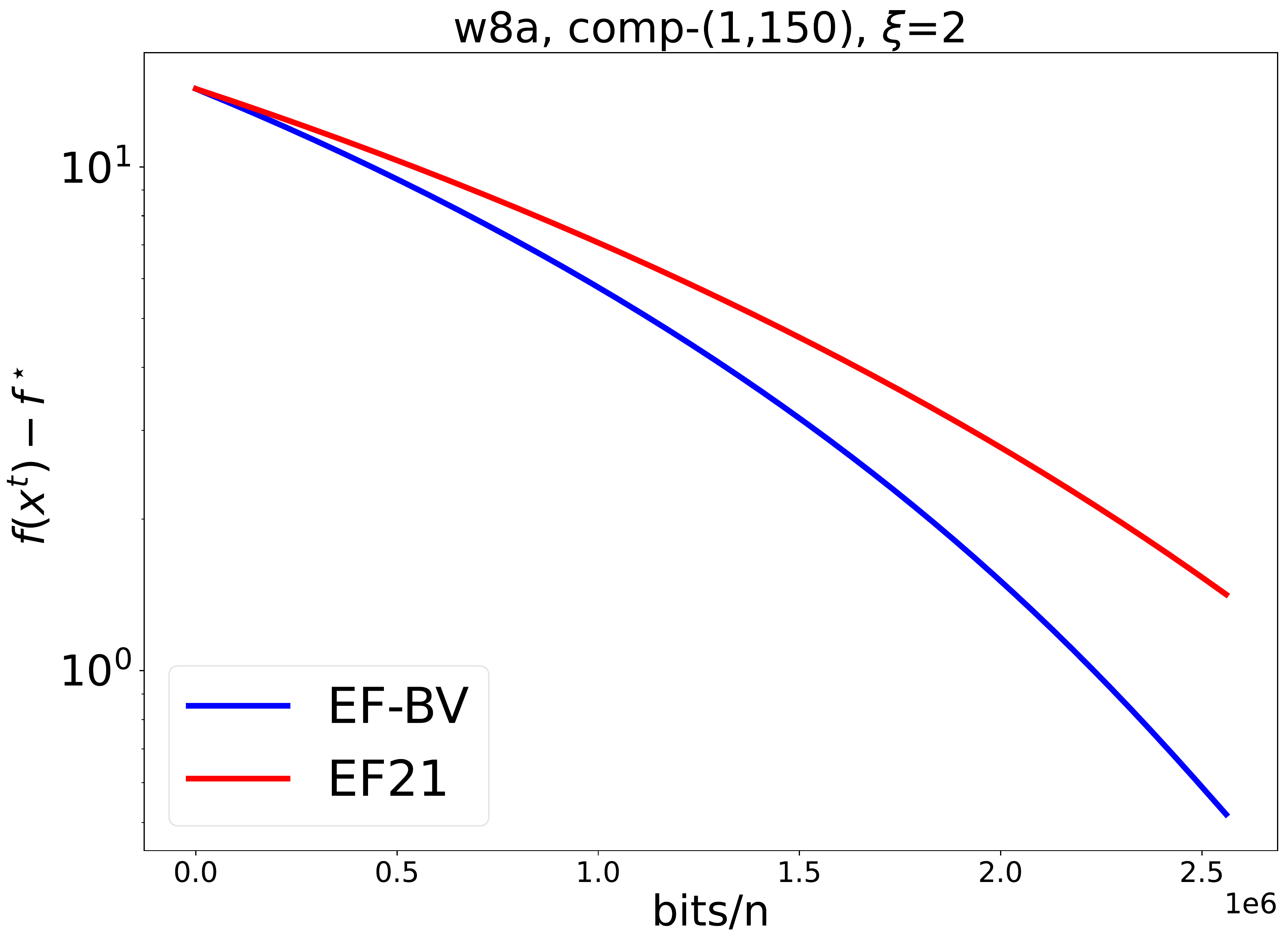}
	\end{subfigure}
		\hfill 
	\begin{subfigure}[b]{0.24\textwidth}
		\centering
		\includegraphics[width=\textwidth]{
		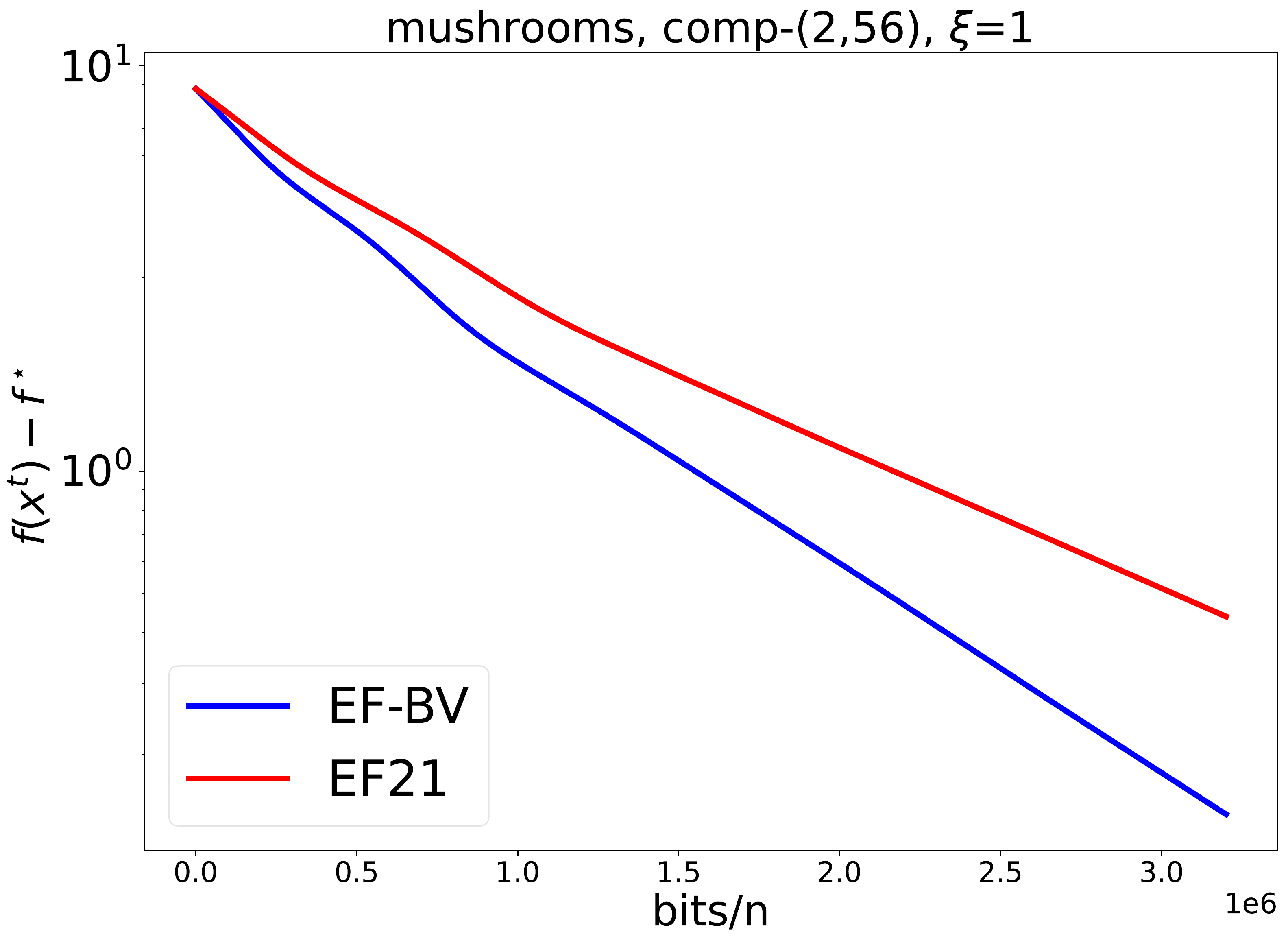}
	\end{subfigure}
	\hfill
	\begin{subfigure}[b]{0.24\textwidth}
		\centering
		\includegraphics[width=\textwidth]{
		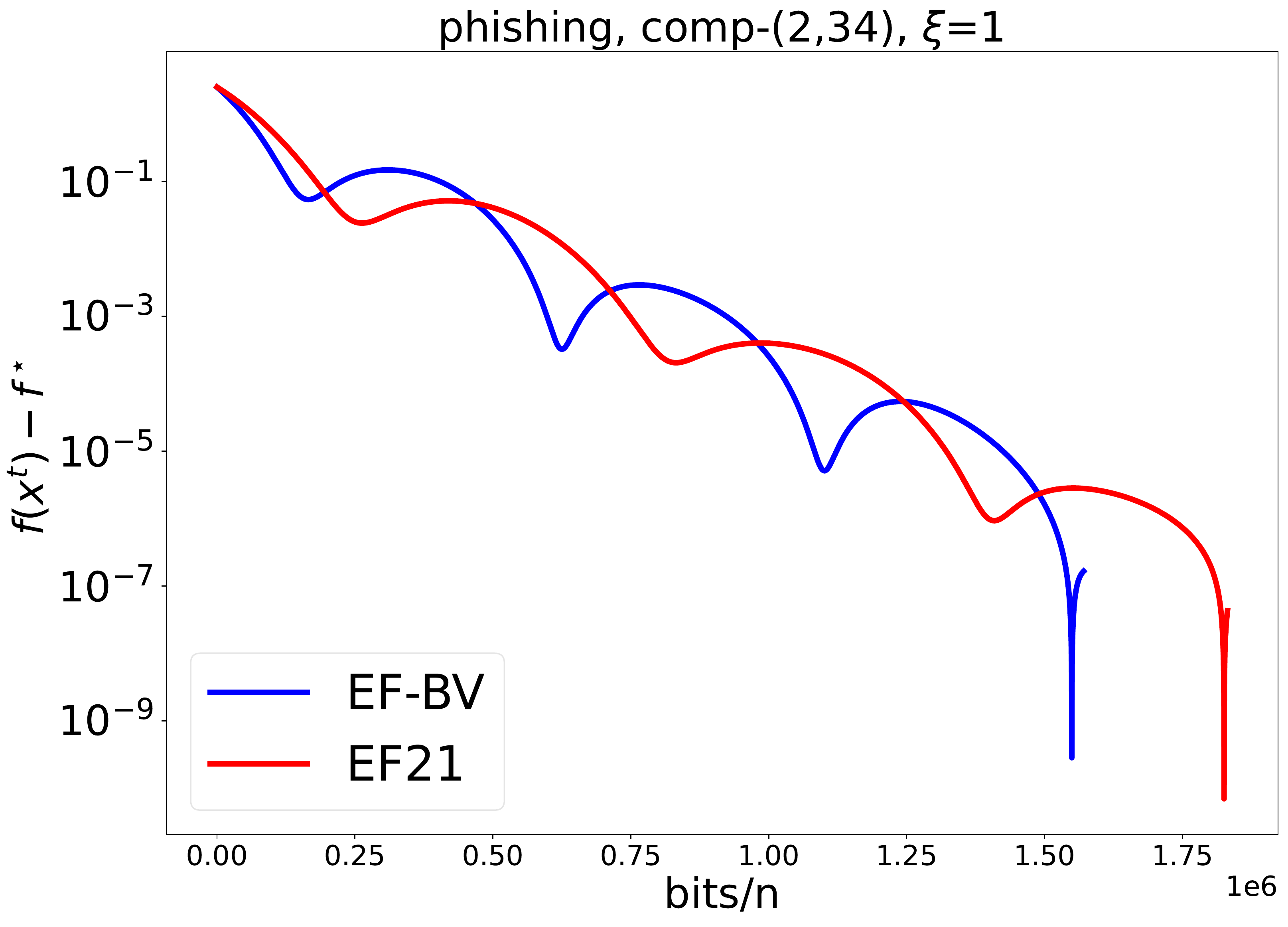}
	\end{subfigure}
	\hfill
	\begin{subfigure}[b]{0.24\textwidth}
		\centering
		\includegraphics[width=\textwidth]{
		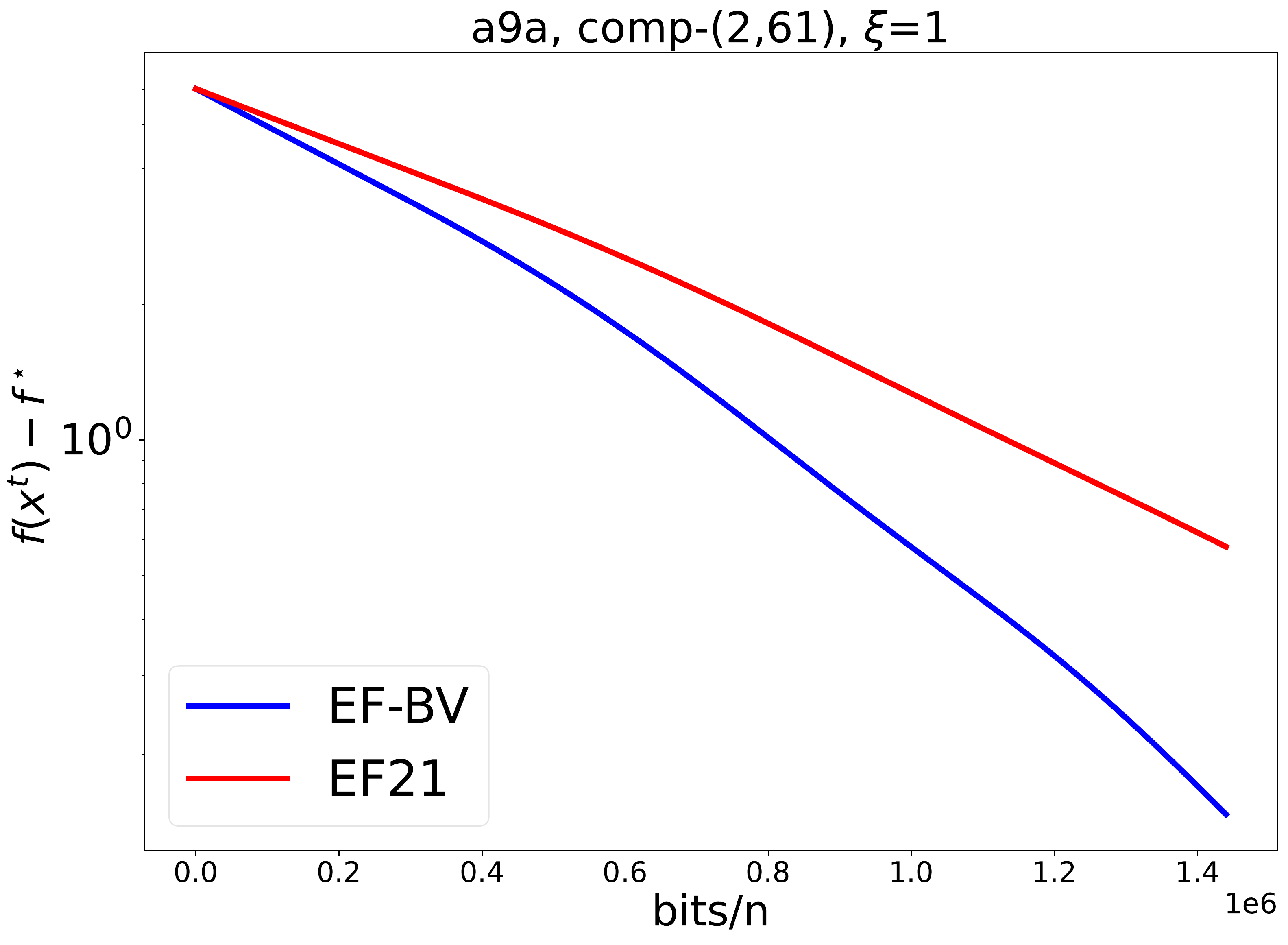}
	\end{subfigure}
	\begin{subfigure}[b]{0.24\textwidth}
		\centering
		\includegraphics[width=\textwidth]{
		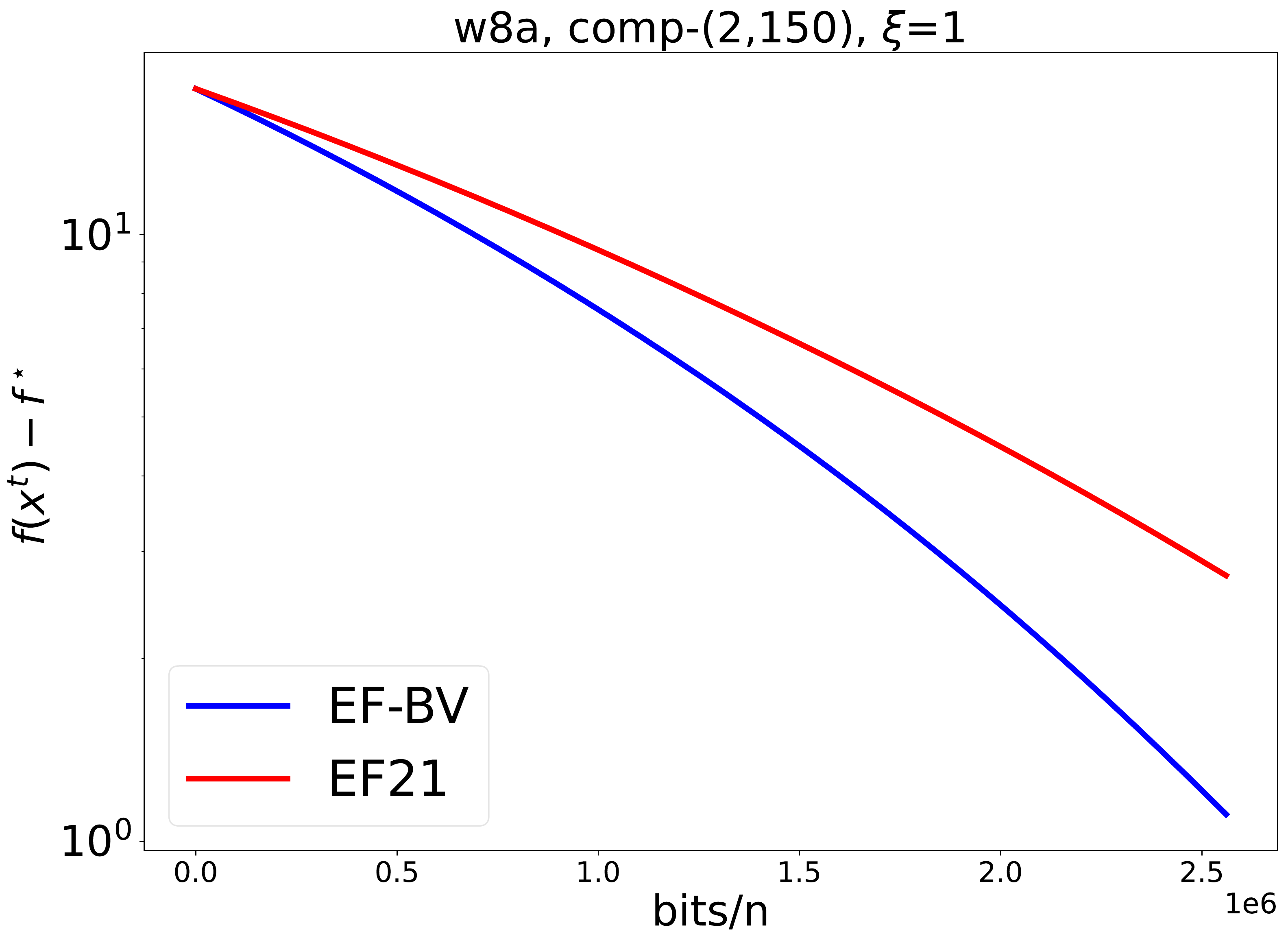}
	\end{subfigure}
	\caption{Experimental results. We plot $f(x^t)-f^\star$ with respect to the number of bits sent by each node during the learning process, which is proportional to $tk$.
	Top row: \texttt{comp-}$(1,d/2)$, overlapping $\xi=1$. Middle row: \texttt{comp-}$(1,d/2)$, overlapping $\xi=2$. Bottom row: \texttt{comp-}$(2,d/2)$, overlapping $\xi=1$.
}\label{fig:0007}
\end{figure}

\clearpage

\bibliographystyle{icml2022}
	\bibliography{IEEEabrv,biblio}

\clearpage
\appendix
\part*{Appendix}

\tableofcontents

\clearpage

\section{New compressors}\label{secappa}

We propose new compressors in our class $\mathbb{C}(\eta,\omega)$.

\subsection{\texttt{mix-}(k,k'): Mixture of \texttt{top-}k and \texttt{rand-}k}

Let $k\in \mathcal{I}_d$ and $k'\in \mathcal{I}_d$, with $k+k'\leq d$. We propose the compressor \texttt{mix-}$(k,k')$. It maps $x\in\mathbb{R}^d$ to $x'\in\mathbb{R}^d$, defined as follows. 
Let $i_1,\ldots,i_k$ be distinct indexes in $\mathcal{I}_d$ such that $|x_{i_1}|,\ldots,|x_{i_k}|$ are the $k$ largest elements of $|x|$ (if this selection is not unique, we can choose any one). These coordinates are kept: $x'_{i_j}=x_{i_j}$, $j=1,\ldots,k$. In addition, $k'$ other coordinates chosen at random in the remaining ones are kept: $x'_{i_j}=x_{i_j}$, $j=k+1,\ldots,k+k'$, where $\{i_j : j=k+1,\ldots,k+k'\}$ is a subset of size $k'$ of $\mathcal{I}_d \backslash \{i_1,\ldots,i_k\}$ chosen uniformly at random. The other coordinates of $x'$ are set to zero.

\begin{proposition}
\label{prop1}\emph{\texttt{mix-}}$(k,k')\in \mathbb{C}(\eta,\omega)$ with $\eta =\frac{d-k-k'}{\sqrt{(d-k)d}}$ and $\omega=\frac{k'(d-k-k')}{(d-k)d}$.
\end{proposition}
As a consequence, \texttt{mix-}$(k,k')\in \mathbb{B}(\alpha)$ with $\alpha=1-\eta^2-\omega = 1-\frac{(d-k-k')^2}{(d-k)d}-\frac{k'(d-k-k')}{(d-k)d}=
\frac{k+k'}{d}$. This is the same $\alpha$ as for \texttt{top-}$(k+k')$ and scaled \texttt{rand-}$(k+k')$.

The proof is given in Appendix~\ref{secproofp4}.

\subsection{\texttt{comp-}(k,k'): Composition of \texttt{top-}k and \texttt{rand-}k}

Let $k\in \mathcal{I}_d$ and $k'\in \mathcal{I}_d$, with $k\leq k'$. We consider the compressor \texttt{comp-}$(k,k')$, proposed in \citet{bar20}, 
which is the composition of \texttt{top-}$k'$ and \texttt{rand-}$k$:
\texttt{top-}$k'$  is applied first, then \texttt{rand-}$k$ is applied to the $k'$ selected (largest) elements. 
That is,  \texttt{comp-}$(k,k')$ maps $x\in\mathbb{R}^d$ to $x'\in\mathbb{R}^d$, defined as follows. Let $i_1,\ldots,i_{k'}$ be distinct indexes in $\mathcal{I}_d$ such that $|x_{i_1}|,\ldots,|x_{i_{k'}}|$ are the $k'$ largest elements of $|x|$ (if this selection is not unique, we can choose any one). Then 
$x'_{i_j}=\frac{k'}{k} x_{i_j}$, $j=1,\ldots,k$, where $\{i_j : j=1,\ldots,k\}$ is a subset of size $k$ of $\{i_1,\ldots,i_{k'}\}$ chosen uniformly at random. The other coordinates of $x'$ are set to zero.

 \texttt{comp-}$(k,k')$ sends $k$ coordinates of its input vector, like \texttt{top-}$k$ and \texttt{rand-}$k$, whatever $k'$. We can note that  \texttt{comp-}$(k,d)={}$\texttt{rand-}$k$ and \texttt{comp-}$(k,k)={}$\texttt{top-}$k$. We have:

\begin{proposition}
\label{prop2}
 \emph{\texttt{comp-}}$(k,k')\in \mathbb{C}(\eta,\omega)$ with $\eta =\sqrt{\frac{d-k'}{d}}$ and $\omega=\frac{k'-k}{k}$.
\end{proposition}

The proof is given in Appendix~\ref{secproofp5}.

\section{New results on \algname{DIANA}}\label{secappb}

We suppose that the compressors $\mathcal{C}_i^t$ are in $\mathbb{C}(\eta,\omega)$, for some $\eta\in[0,1)$ and $\omega\geq 0$. Viewing  \algname{DIANA} as \algname{EF-BV} with $\nu=1$, we define $r$, $s^\star$, $\theta^\star$ as before, as well as 
$r_{\mathrm{av}} \eqdef \eta^2+\oma$. 
We obtain, as corollaries of Theorems \ref{theo1} and \ref{theo2}:

\begin{theorem}\label{coro1}Suppose that $R=0$ and $f$ satisfies the P{\L}  condition with some constant  $\mu>0$. 
In  \algname{DIANA}, suppose that $\lambda \in (0,1]$ is such that $r<1$, and
\begin{equation*}
0<\gamma \leq \frac{1}{L+\tilde{L}\sqrt{\frac{r_{\mathrm{av}}}{r}}\frac{1}{s^\star}}.
\end{equation*}
For every $t\geq 0$, define the Lyapunov function
$\Psi^t \eqdef f(x^t)-f^\star + \frac{\gamma}{2\theta^\star}  \frac{1}{n}\sum_{i=1}^n \sqnorm{\nabla f_i(x^t)-h_i^{t}}$, 
where $f^\star \eqdef f(x^\star)$, for any minimizer $x^\star$ of $f$. 
Then, for every $t\geq 0$,
\begin{align*}
\Exp{\Psi^{t}} 
&\leq \left(\max\left(1-\gamma\mu, {\frac{r+1}{2}}\right) \right)^t\Psi^0.
\end{align*}
\end{theorem}

\begin{theorem}\label{coro2}
Suppose that $f+R$ satisfies the  the K{\L}  condition with some constant $\mu>0$. 
In  \algname{DIANA}, suppose that $\lambda \in (0,1]$ is such that $r<1$, and
\begin{equation*}
0<\gamma \leq \frac{1}{2L+\tilde{L}\sqrt{\frac{r_{\mathrm{av}}}{r}}\frac{1}{s^\star}}.
\end{equation*}
$\forall t\geq 0$, define the Lyapunov function
$\Psi^t \eqdef f(x^t)+R(x^t)-f^\star - R^\star  + \frac{\gamma}{2\theta^\star}  \frac{1}{n}\sum_{i=1}^n \sqnorm{\nabla f_i(x^t)-h_i^{t}}$,
where $f^\star \eqdef f(x^\star)$ and $R^\star \eqdef R(x^\star)$, for any minimizer $x^\star$ of $f+R$. 
Then, for every $t\geq 0$,
\begin{align*}
\Exp{\Psi^{t}}  &\leq \left(\max\left({\frac{1}{1+\frac{1}{2}\gamma\mu}},\frac{r+1}{2}\right)\right)^t\Psi^0.
\end{align*}
\end{theorem}\smallskip

Interestingly,  \algname{DIANA}, used beyond its initial setting with compressors in $\mathbb{B}(\alpha)$ with $\lambda=1$, just reverts to (the original) 
 \algname{EF21}, as shown in Fig.~\ref{fig1}. This shows how our unified framework reveals connections between these two algorithms and  unleashes their potential.

\section{Experiments}\label{appexp}

\subsection{Datasets and experimental setup}

We consider the heterogeneous data distributed regime, which means that all parallel nodes store different data points, but use the same type of learning function. We adopt the datasets from LibSVM~\citep{chang2011libsvm} and  we split them, after random shuffling, into 
$n\leq N$ 
blocks, where $N$ is the total number of data points (the left-out data points from the integer division of $N$ by $n$ are stored at the last node). The corresponding values are shown in Tab.~\ref{tab:logistic_datasets}. 
To make our  setting more realistic, we consider that different nodes partially share some data:
we set the overlapping factor to be $\xi\in\{1, 2\}$, where $\xi=1$ means no overlap and $\xi=2$ means that the data is partially shared among the nodes, with a redundancy factor of 2; this is achieved by  sequentially assigning 2 blocks of data to every node.
The experiments were conducted using 24 NVIDIA-A100-80G GPUs, each with  
80GB memory. 

\begin{table}[b]
 \caption{Values of $d$ and $N$ for the considered datasets.}
    \label{tab:logistic_datasets}
    \centering
    \begin{tabular}{c|c|c}
        \toprule
        \multirow{2}{*}{Dataset} & \multirow{2}{*}{$N$ (total \# of datapoints)} & \multirow{2}{*}{$d$ (\# of features)} \\
        ~ & ~ &~ 
        \\\hline
        \multirow{1}{*}{\texttt{mushrooms}} & \multirow{1}{*}{8,124} & \multirow{1}{*}{112} 
        \\ \hline
        \multirow{1}{*}{\texttt{phishing}} & \multirow{1}{*}{11,055} & \multirow{1}{*}{68} 
        \\ \hline
        \multirow{1}{*}{\texttt{a9a}} & \multirow{1}{*}{32,561} & \multirow{1}{*}{123} 
        \\ \hline
        \multirow{1}{*}{\texttt{w8a}} & \multirow{1}{*}{49,749} & \multirow{1}{*}{300} 
        \\  
        \bottomrule
    \end{tabular}
   \end{table}

We consider logistic regression, which consists in minimizing the $\mu$-strongly convex function 
\begin{equation*}
f=\frac{1}{n} \sum_{i=1}^{n}f_i,
\end{equation*}
with, for every $i\in\mathcal{I}_n$,
\begin{equation*}
        f_i(x) =\frac{1}{N_i} \sum_{j=1}^{N_i} \log\!\Big(1+\exp\!\left(-b_{i,j} x^{\top} a_{i,j}\right)\Big)+ \frac{\mu}{2} \|x\|^2, 
    \end{equation*}
where $\mu$, set to $0.1$, 
is the strong convexity constant; $N_i$ is the number of data points at node $i$; the $a_{i,j}$ are the training vectors and the $b_{i,j} \in\{-1,1\}$ the corresponding labels. 
Note that there is no regularizer in this problem; that is, $R=0$.

We set $L=\tilde{L}=\sqrt{\sum_{i=1}^n L_i^2}$, with $L_i =\mu +  \frac{1}{4N_i} \sum_{j=1}^{N_i} \|a_{i,j}\|^2$. We use independent compressors of type \texttt{comp-}$(k,k')$ at every node, for some small $k$ and large $k'<d$. These compressors are biased ($\eta>0$) and have a variance $\omega>1$, so they are not contractive: they don't belong to $\mathbb{B}(\alpha)$ for any $\alpha$. We have $\oma=\frac{\omega}{n}$.
Thus, we place ourselves in the conditions of Theorem~\ref{theo1}, and we compare \algname{EF-BV} with 
\begin{equation*}
\lambda=\lambda^\star,\quad \nu=\nu^\star,\quad \gamma=\frac{1}{L+\tilde{L}\sqrt{\frac{r_{\mathrm{av}}}{r}}\frac{1}{s^\star}}
\end{equation*}
to \algname{EF21}, which corresponds to the particular case of \algname{EF-BV} with
\begin{equation*}
\nu = \lambda=\lambda^\star,\quad \gamma=\frac{1}{L+\tilde{L}\frac{1}{s^\star}}.
\end{equation*}

\begin{table}[t]
	\caption{Parameter values of \algname{EF-BV} and \algname{EF21} in the different settings.  $k'$ in \texttt{comp-}$(k, k')$ is set to $d/2$ and $n=1000$. In pairs of values like (1,2), the first value is $k$ and the second value is $\xi$.}\label{tab10}
	\centering
	\resizebox{1.0\textwidth}{!}{
	\begin{tabular}{c|c|ccc|ccc|ccc|ccc}
		\toprule
		\multirow{2}{*}{Method} & \multirow{2}{*}{Params} & \multicolumn{3}{c}{mushrooms} & \multicolumn{3}{c}{phishing} & \multicolumn{3}{c}{a9a} & \multicolumn{3}{c}{w8a}\\ \cmidrule{3-5} \cmidrule{6-8} \cmidrule{9-11} \cmidrule{12-14}
		~ & ~ & (1,1) & (1,2) & (2,1) & (1,1) & (1,2) & (2,1) & (1,1) & (1,2) & (2,1) & (1,1) & (1,2) & (2,1)\\ \hline
		 & \multirow{2}{*}{$\eta$} & \multirow{2}{*}{0.707} & \multirow{2}{*}{0.707} & \multirow{2}{*}{0.707} & \multirow{2}{*}{0.707} & \multirow{2}{*}{0.707} & \multirow{2}{*}{0.707} & \multirow{2}{*}{0.710} & \multirow{2}{*}{0.710} & \multirow{2}{*}{0.710}  & \multirow{2}{*}{0.707} & \multirow{2}{*}{0.707} & \multirow{2}{*}{0.707} \\
		 & ~ & ~ & ~\\ \hline

		& \multirow{2}{*}{$\omega$} & \multirow{2}{*}{55} & \multirow{2}{*}{55} & \multirow{2}{*}{27} & \multirow{2}{*}{33} & \multirow{2}{*}{33} & \multirow{2}{*}{16} & \multirow{2}{*}{60} & \multirow{2}{*}{60} & \multirow{2}{*}{29.5} & \multirow{2}{*}{149} & \multirow{2}{*}{149} & \multirow{2}{*}{74}\\
		 & ~ & ~ & ~\\ \hline  

		 & \multirow{2}{*}{$\omega_{\mathrm{av}}$} & \multirow{2}{*}{0.055} & \multirow{2}{*}{0.055} & \multirow{2}{*}{0.027} & \multirow{2}{*}{0.033} & \multirow{2}{*}{0.033} & \multirow{2}{*}{0.016}  & \multirow{2}{*}{0.06} & \multirow{2}{*}{0.06} & \multirow{2}{*}{0.295} & \multirow{2}{*}{0.149} & \multirow{2}{*}{0.149} & \multirow{2}{*}{0.074}\\
		 & ~ & ~ & ~\\ \hline 

		EF-BV & \multirow{2}{*}{$\lambda$} & 5.32e-3 & 5.32e-3 & 1.08e-2 & 8.85e-3 & 8.85e-3 & 1.82e-2 & 4.83e-3 & 4.83e-3 & 9.8e-3 & 1.96e-3 & 1.96e-3 & 3.95e-3\\
		EF21 & ~ & 5.32e-3 & 5.32e-4 & 1.08e-2 & 8.85e-3 & 8.85e-3 & 1.82e-2 & 4.83e-3 & 4.83e-3 & 9.8e-3 & 1.96e-3 & 1.96e-3 & 3.95e-3\\ \hline

		EF-BV & \multirow{2}{*}{$\nu$} & 1 & 1 & 1 & 1 & 1 & 1 & 1 & 1 & 1 & 1 & 1 & 1 \\
		EF21 & ~ & 5.32e-3 & 5.32e-4 & 1.08e-2 & 8.85e-3 & 8.85e-3 & 1.82e-2 & 4.83e-3 & 4.83e-3 & 9.8e-3 & 1.96e-3 & 1.96e-3 & 3.95e-3\\ \hline

		EF-BV & \multirow{2}{*}{$r$} & 0.998 & 0.998 & 0.997& 0.997 & 0.997 & 0.994 & 0.999 & 0.999 & 0.997 & 0.999 & 0.999 & 0.999\\
		EF21 & ~ & 0.998 & 0.998 & 0.997 & 0.997 & 0.997 & 0.994 &0.999 & 0.999 & 0.997 & 0.999 & 0.999 & 0.999\\ \hline

		EF-BV & \multirow{2}{*}{$r_{\mathrm{av}}$} & 0.555 & 0.555 & 0.527 & 0.533 & 0.533 & 0.516 & 0.564 & 0.564 & 0.534 & 0.649 & 0.649 & 0.574\\
		EF21 & ~ & 0.998 & 0.998 & 0.997 & 0.997 & 0.997 & 0.994 & 0.999 & 0.999 & 0.997 & 0.999 & 0.999 & 0.999\\ \hline

		EF-BV & \multirow{2}{*}{$\sqrt{\frac{r_{\mathrm{av}}}{r}}$} & 0.746 & 0.746 & 0.727 & 0.731 & 0.731 & 0.720 & 0.752 & 0.752 & 0.731 & 0.806 & 0.806 & 0.758\\
		EF21 & ~ & 1 & 1 & 1 & 1 & 1 & 1 & 1 & 1 & 1 & 1 & 1 & 1\\ \hline  

		EF-BV & \multirow{2}{*}{$s^\star$} & 3.90e-4 & 3.90e-4 & 7.94e-4 & 6.50e-4 & 6.50e-4 & 1.34e-3 & 3.5e-4 & 3.5e-4 & 7.13e-4 & 1.44e-4 & 1.44e-4 & 2.90e-4\\
		EF21 & ~ & 3.90e-4 & 3.90e-4 & 7.94e-4 & 6.50e-4 & 6.50e-4 & 1.34e-3 & 3.5e-4 & 3.5e-4 & 7.13e-4 & 1.44e-4 & 1.44e-4 & 2.90e-4\\ \hline

		EF-BV & \multirow{2}{*}{$\gamma$} & 1.38e-4 & 1.43e-4 & 2.87e-4 & 2.33e-3 & 2.36e-3 & 4.80e-3 & 2.53e-4 & 2.58e-4 & 5.28e-4 & 1.01e-4 & 1.15e-4 & 2.15e-4\\
		EF21 & ~ & 1.03e-4 & 1.06e-4 & 2.10e-4 & 1.71e-3 & 1.73e-3 & 3.49e-3 & 1.91e-4 & 1.84e-4 & 3.87e-4 & 8.12e-5 & 9.31e-5 & 1.63e-4\\ \hline
	\end{tabular}}
\end{table}

\subsection{Experimental results and analysis}  

We show in Fig.~\ref{fig:0007} the results with $k=1$ or $k=2$ in the compressors  \texttt{comp-}$(k,k')$, and overlapping factor $\xi=1$ or $\xi=2$.
We chose $k'=\frac{d}{2}$ and $n=1000$. The corresponding values of $\eta$, $\omega$, $\oma$, and the parameter values used in the algorithms are shown in Tab.~\ref{tab10}. We can see that there is essentially no difference between the two choices $\xi=1$ and $\xi=2$, and the qualitative behavior for $k=1$ and $k=2$ is similar. Thus, we observe that  \algname{EF-BV} converges always faster than \algname{EF21}; this is consistent with our analysis. 

We tried other values of $n$, including the largest value $n=N$, for which there is only one data point at every node. The behavior of \algname{EF21} and \algname{EF-BV} was the same as for $n=1000$, so we don't show the results.

We tried other values of $k'$. The behavior of \algname{EF21} and \algname{EF-BV} was the same as for $k'= \frac{d}{2}$ overall, so we don't show the results. We noticed that the difference between the two algorithms was smaller when $k'$ was smaller; this is expected, since for $k'=k$, the compressors revert to \texttt{top-}$k$, for which \algname{EF21} and \algname{EF-BV} are the same algorithm.

To sum up, the experiments confirm our analysis: when $\omega$ and $n$ are large,  so that the key factor $\sqrt{\frac{r_{\mathrm{av}}}{r}}$ is small, randomness is exploited in \algname{EF-BV}, with larger values of $\nu$ and $\gamma$ allowed than in \algname{EF21}, and this yields faster convergence. 

In future work, we will design and compare other compressors in our new class $\mathbb{C}(\eta,\omega)$, performing well in both homogeneous and heterogeneous regimes.

\subsection{Additional experiments in the nonconvex setting} 

We consider the logistic regression  problem with a nonconvex regularizer:
\begin{equation}
   f(x)=\frac{1}{n} \sum_{i=1}^{n} \log \left(1+\exp \left(-y_{i} a_{i}^{\top} x\right)\right)+\lambda \sum_{j=1}^{d} \frac{x_{j}^{2}}{1+x_{j}^{2}}, 
\end{equation}
where $a_{i} \in \mathbb{R}^{d}, y_{i} \in\{-1,1\}$ are the training data, and $\lambda>0$ is the regularizer parameter. We used $\lambda=0.1$ in all experiments. We present the results in Fig.~\ref{fig13}.
  
\begin{figure}[!htbp]
   \centering
   \begin{subfigure}[b]{0.32\textwidth}
      \centering
      \includegraphics[width=\textwidth]{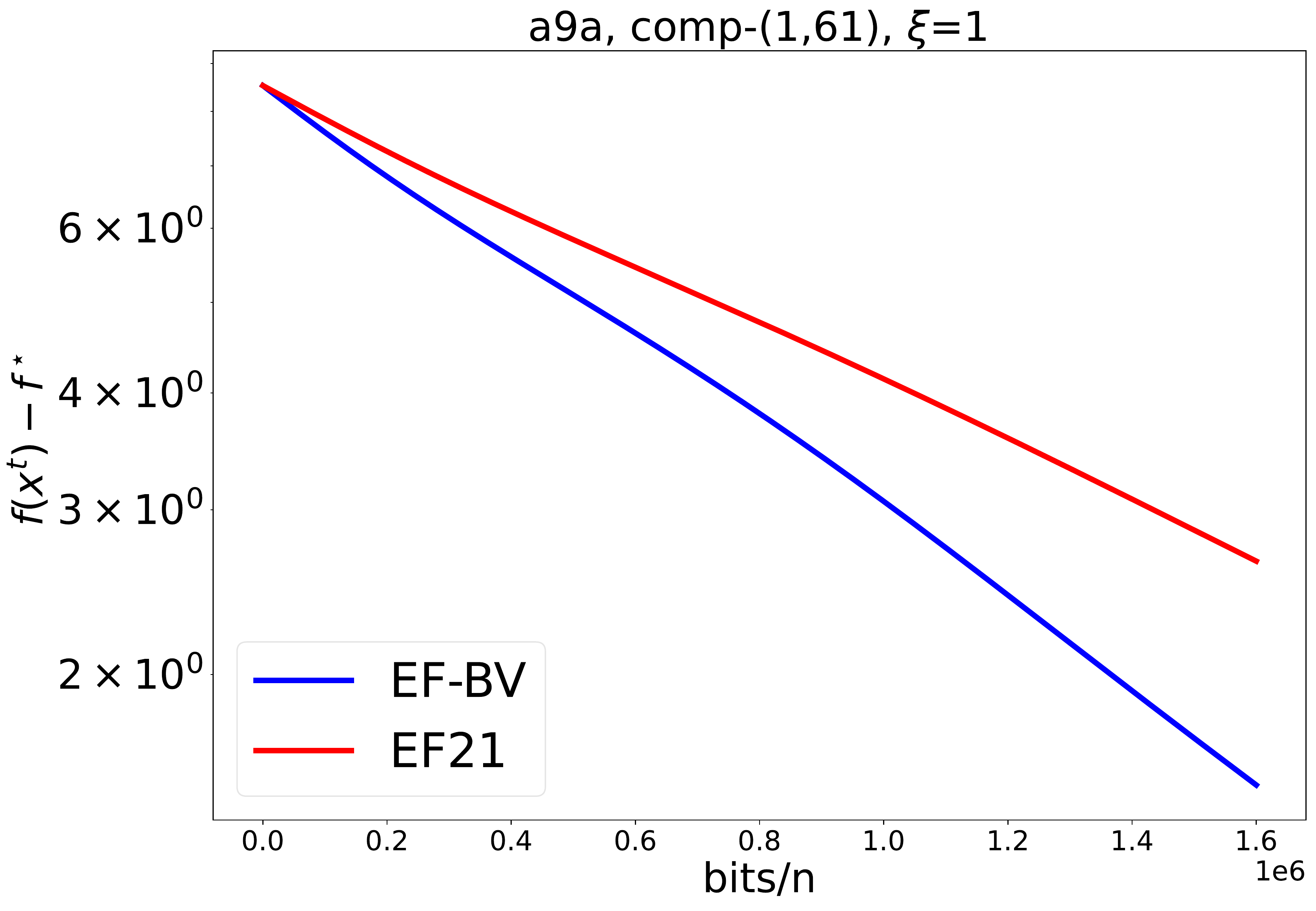}
   \end{subfigure}
   \hfill
   \begin{subfigure}[b]{0.32\textwidth}
      \centering
      \includegraphics[width=\textwidth]{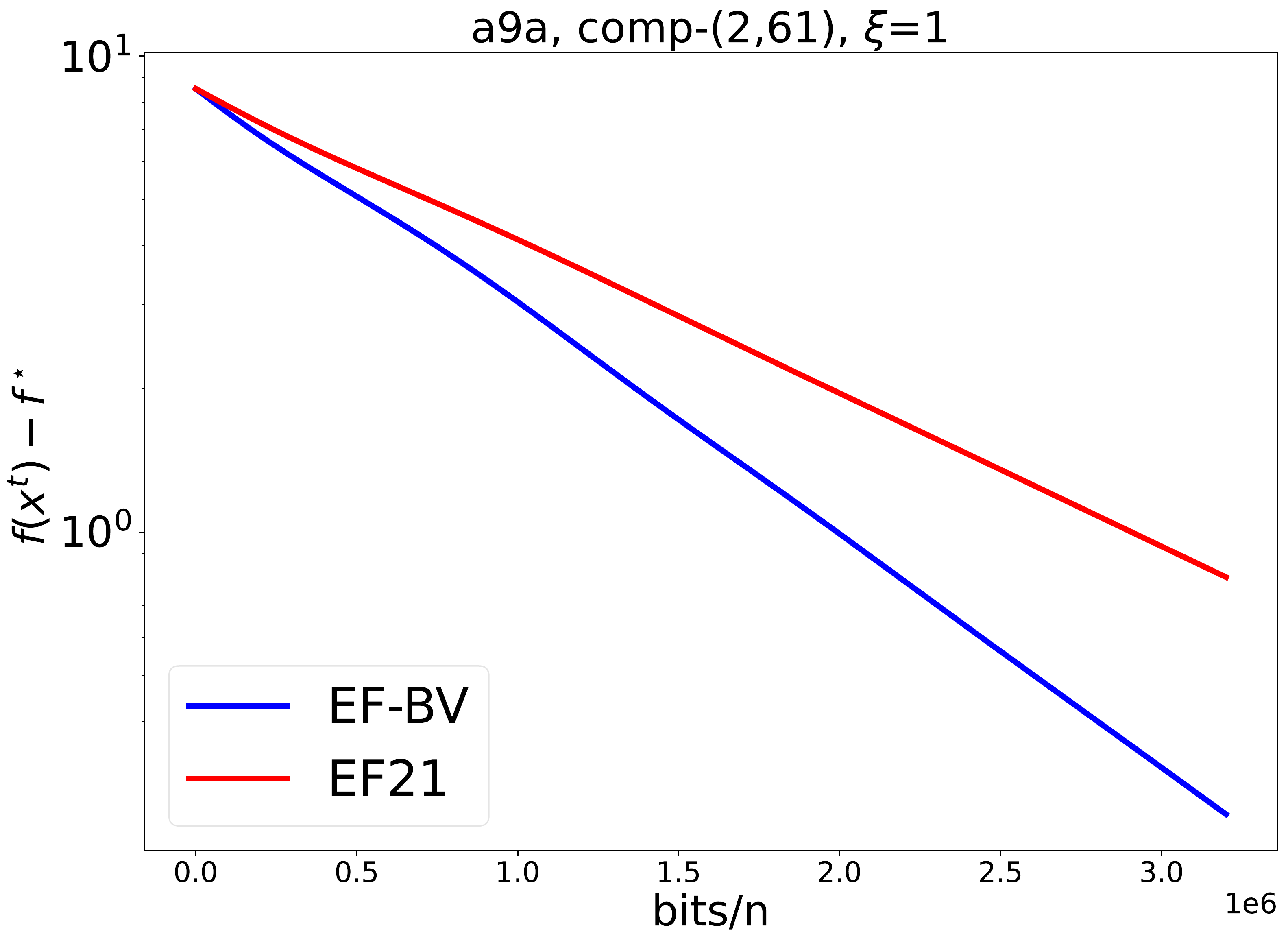}
   \end{subfigure}
   \hfill
   \begin{subfigure}[b]{0.32\textwidth}
      \centering
      \includegraphics[width=\textwidth]{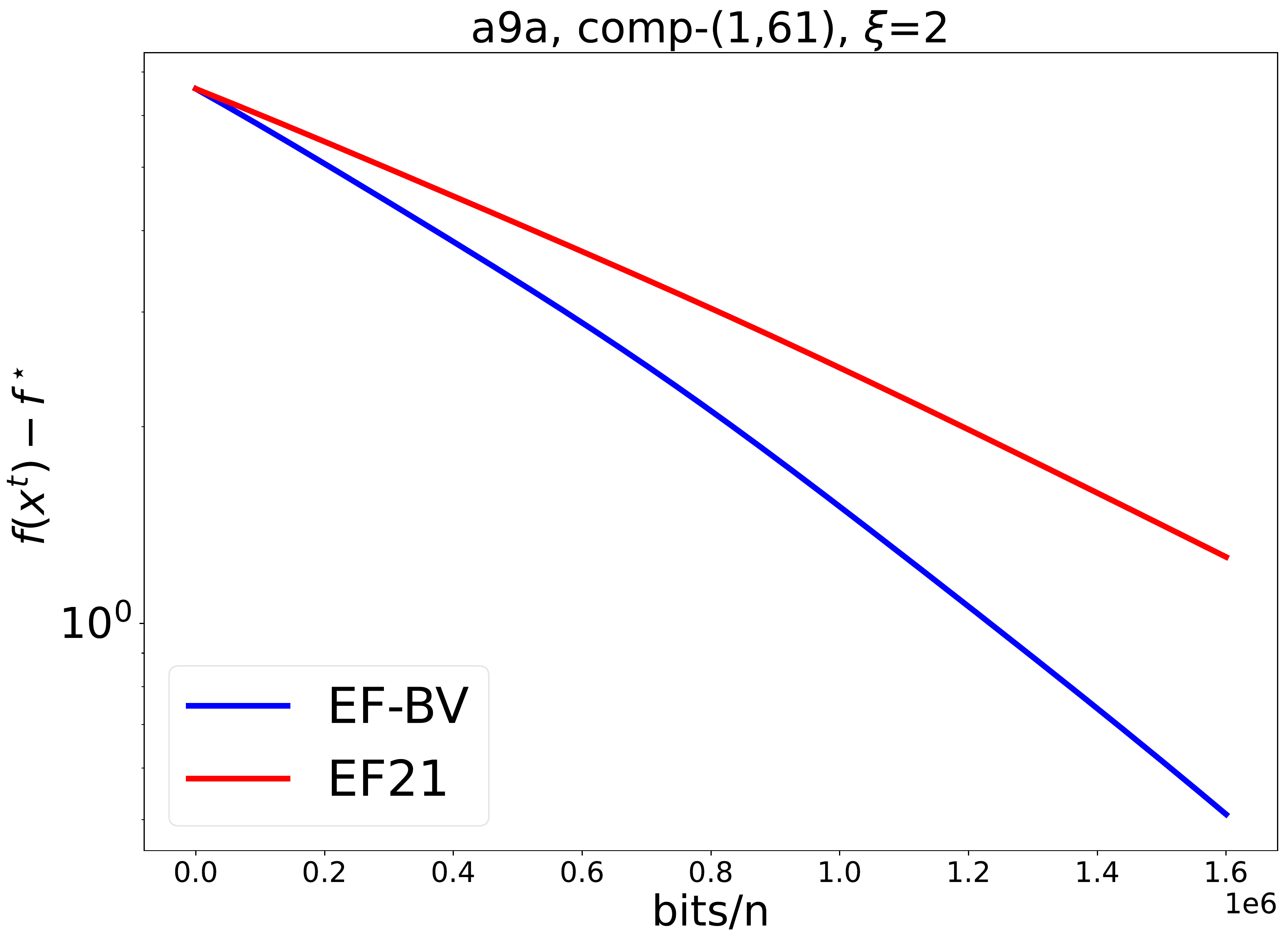}
   \end{subfigure}
   \hfill
   \begin{subfigure}[b]{0.32\textwidth}
      \centering
      \includegraphics[width=\textwidth]{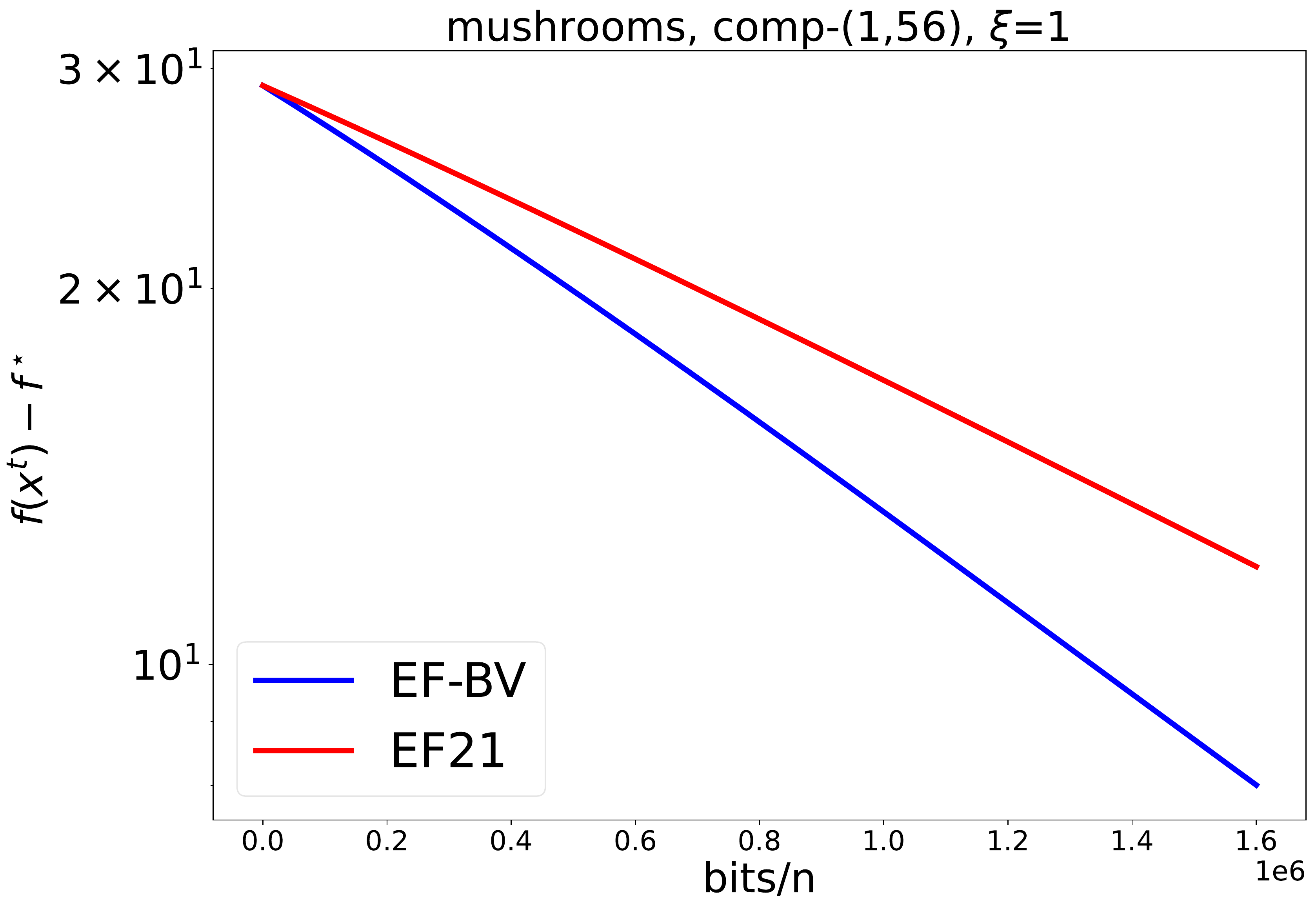}
   \end{subfigure}
   \hfill
   \begin{subfigure}[b]{0.32\textwidth}
      \centering
      \includegraphics[width=\textwidth]{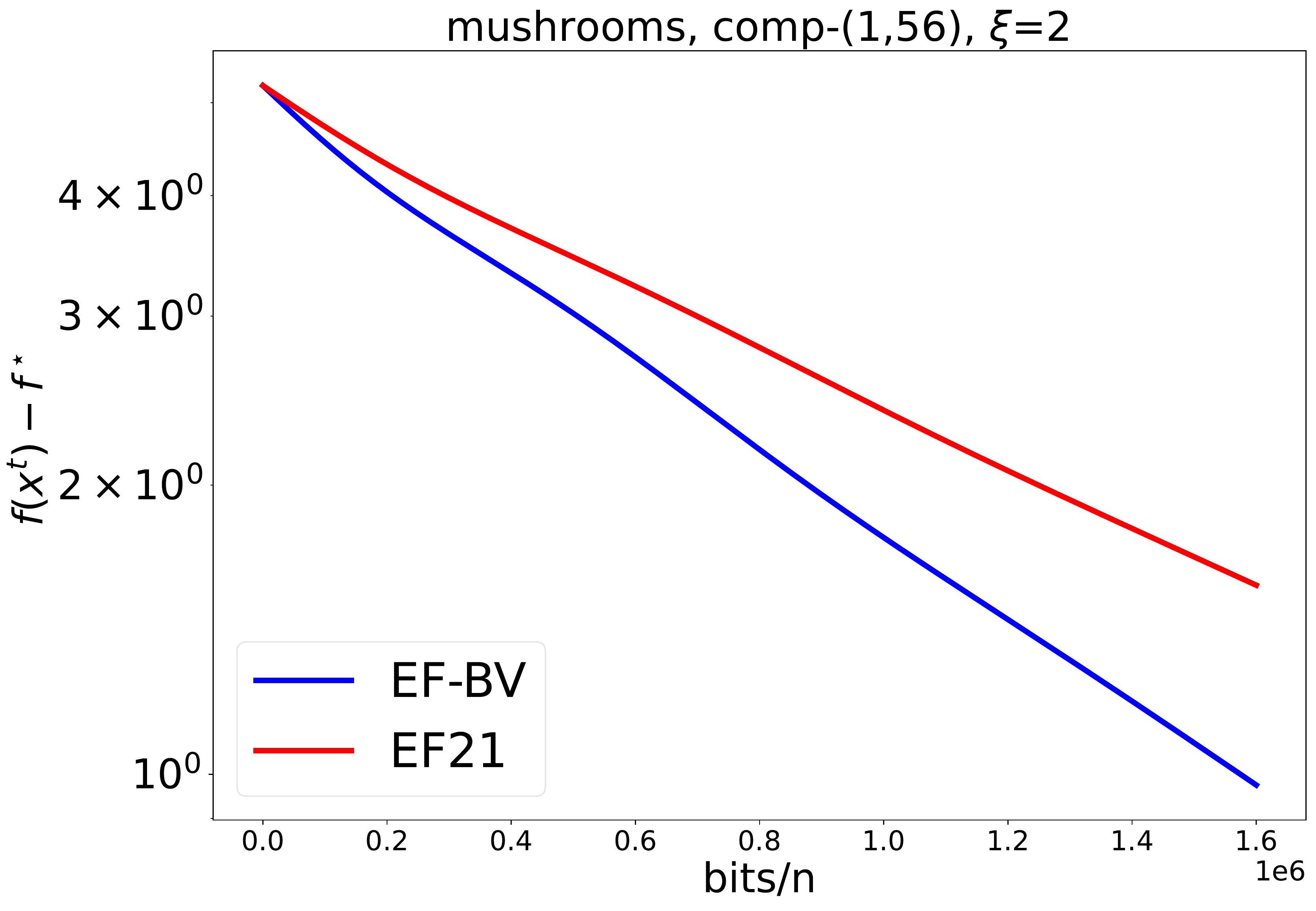}
   \end{subfigure}
   \hfill
   \begin{subfigure}[b]{0.32\textwidth}
      \centering
      \includegraphics[width=\textwidth]{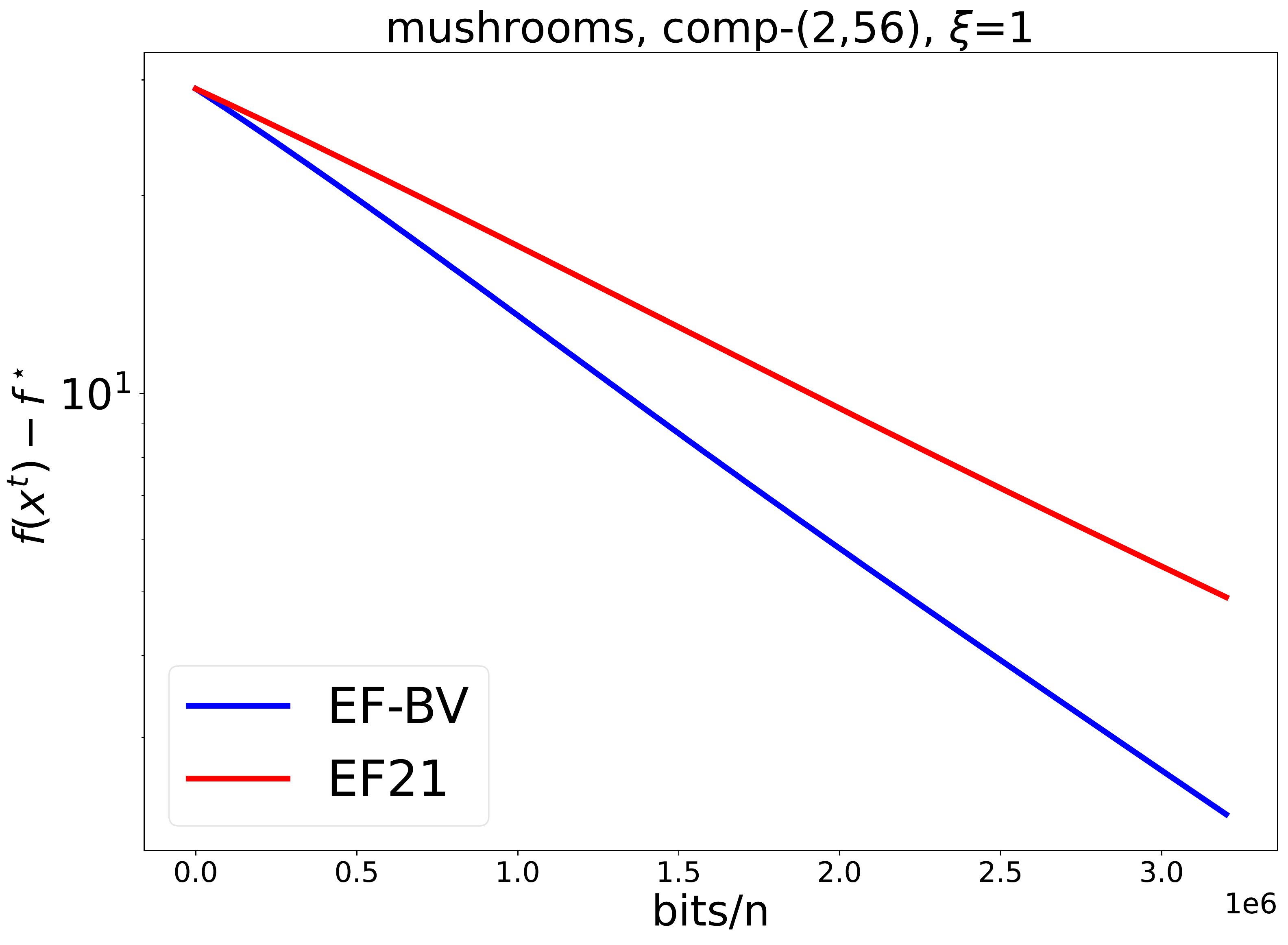}
   \end{subfigure}
   \hfill
   \begin{subfigure}[b]{0.32\textwidth}
      \centering
      \includegraphics[width=\textwidth]{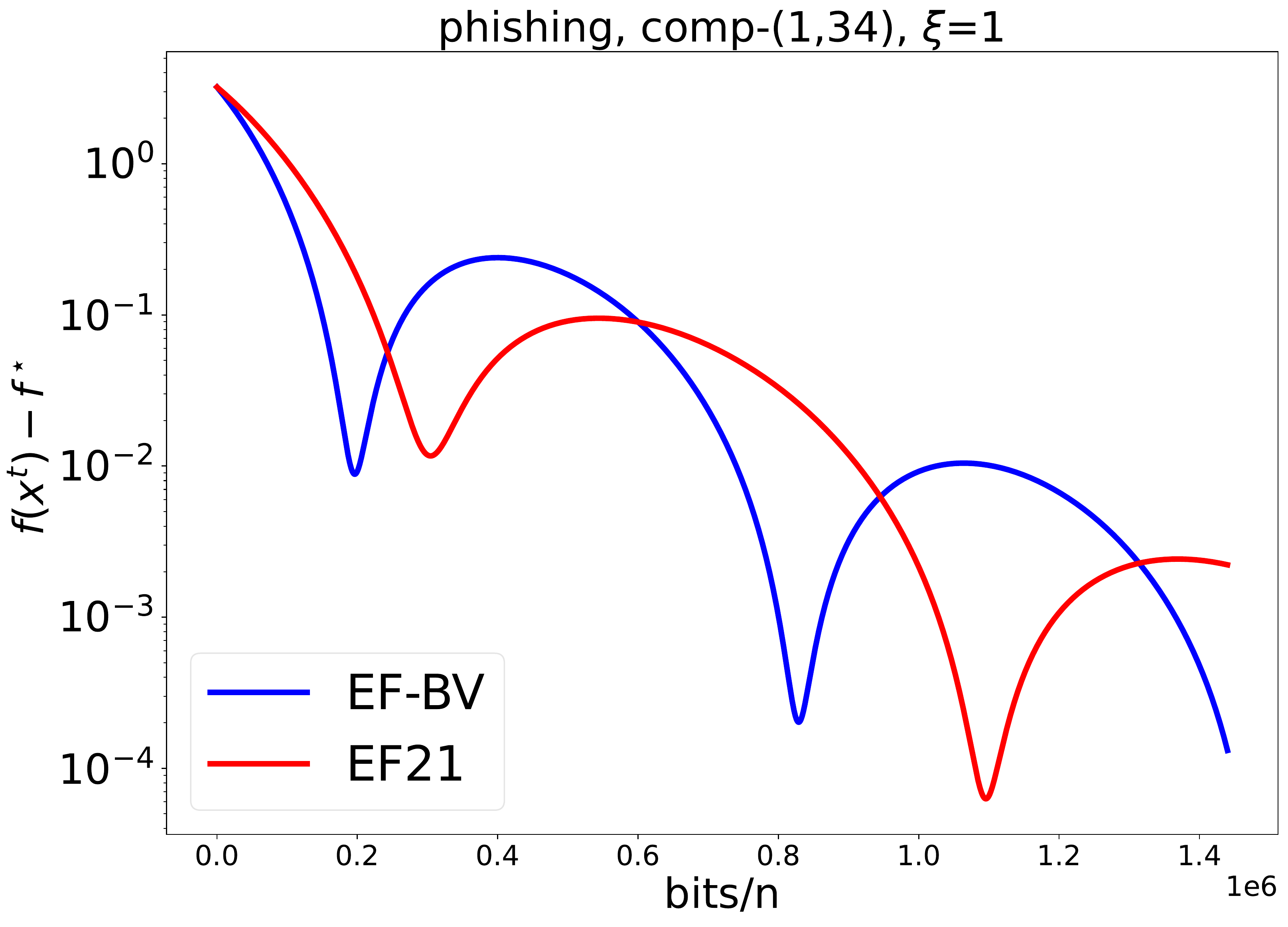}
   \end{subfigure}
   \hfill
   \begin{subfigure}[b]{0.32\textwidth}
      \centering
      \includegraphics[width=\textwidth]{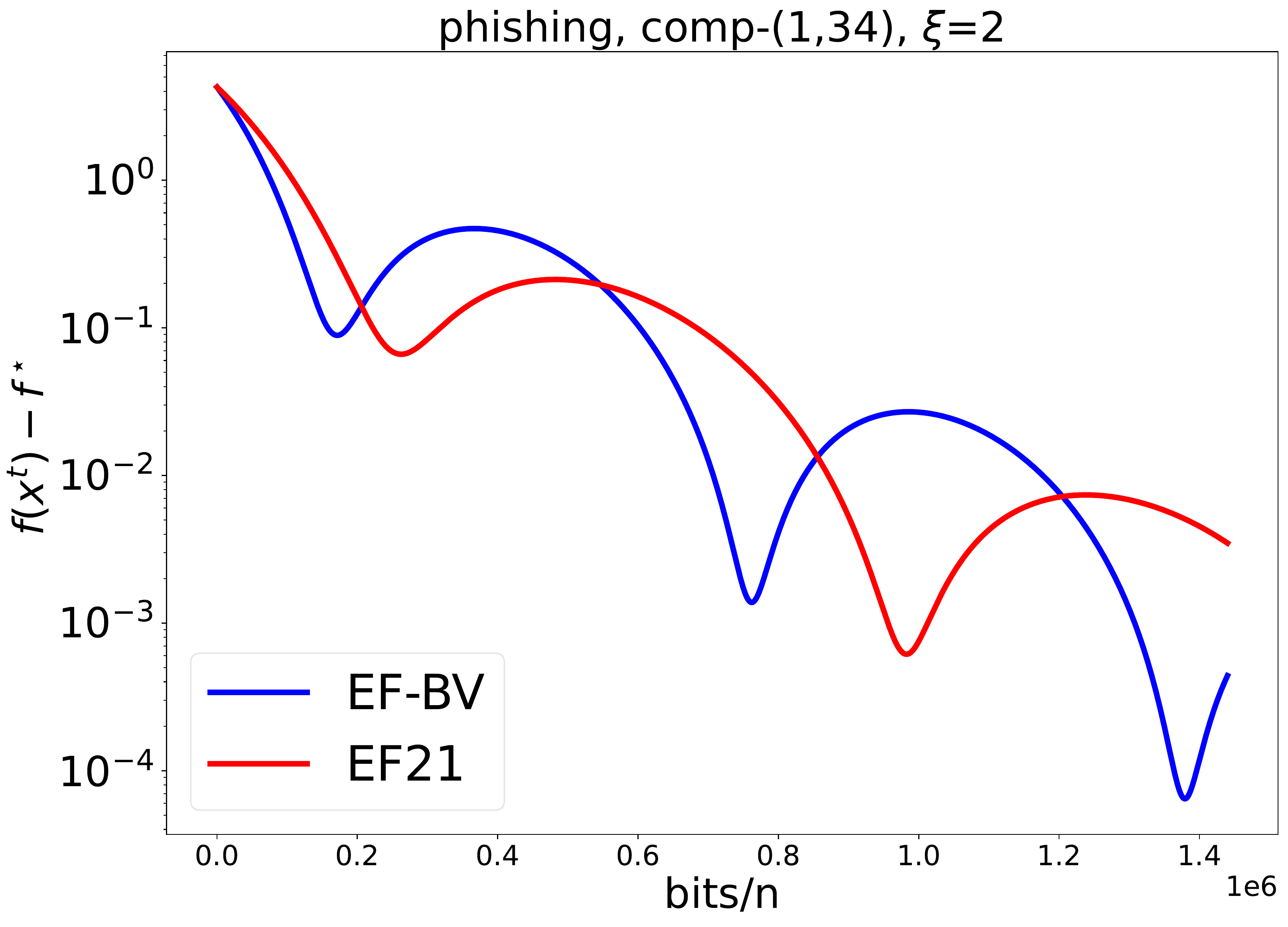}
   \end{subfigure}
   \hfill
   \begin{subfigure}[b]{0.32\textwidth}
      \centering
      \includegraphics[width=\textwidth]{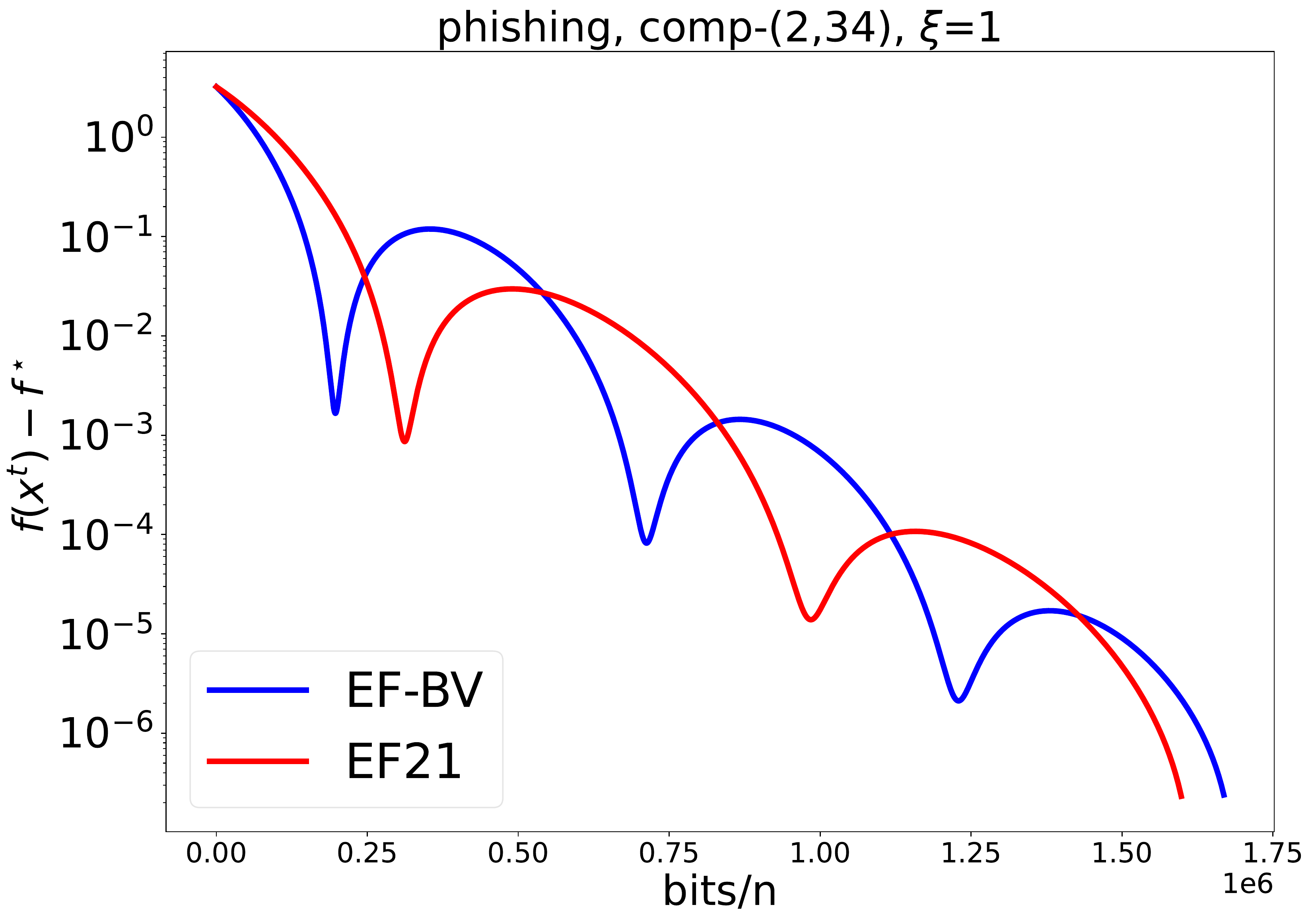}
   \end{subfigure}
      \caption{Comparison between \algname{EF21} and \algname{EF-BV} in the nonconvex setting. We see that \algname{EF-BV} outperforms \algname{EF21} for all datasets.}
      \label{fig13}
\end{figure}

\section{Proof of Proposition~\ref{prop1}}\label{secproofp4}

We first calculate $\omega$. Let $x\in\mathbb{R}^d$.
\begin{align*}
\big\|\mathcal{C}(x)- \mathbb{E}[\mathcal{C}(x)]\big\|^2&=\sum_{i\in \mathcal{I}_d\backslash \{i_1,\ldots, i_{k+k'}\}} \left(\frac{k'}{d-k}\right)^2|x_i|^2
+ \sum_{j=k+1}^{k+k'} \left(\frac{d-k-k'}{d-k}\right)^2|x_{i_j}|^2.
\end{align*}
Therefore, by taking the expectation over the random indexes $i_{k+1},\ldots,i_{2k}$,
\begin{align*}
\Exp{\big\|\mathcal{C}(x)- \mathbb{E}[\mathcal{C}(x)]\big\|^2} &=\sum_{i\in \mathcal{I}_d\backslash \{i_1,\ldots, i_{k}\}} 
\left( \frac{d-k-k'}{d-k}\left(\frac{k'}{d-k}\right)^2 +\frac{k'}{d-k} \left(\frac{d-k-k'}{d-k}\right)^2 
\right)|x_i|^2\\
&=\frac{k'(d-k-k')}{(d-k)^2}\sum_{i\in \mathcal{I}_d\backslash \{i_1,\ldots, i_{k}\}}  |x_i|^2.
\end{align*}
Moreover, since the $|x_{i_j}|$ are the largest elements of $|x|$, for every $j=1,\ldots,k$, 
\begin{equation*}
|x_{i_j}|^2\geq \frac{1}{d-k}\sum_{i\in \mathcal{I}_d\backslash \{i_1,\ldots, i_{k}\}}  |x_i|^2,
\end{equation*}
so that
\begin{equation*}
\|x\|^2 = \sum_{i\in \mathcal{I}_d}  |x_i|^2 \geq \left(1+\frac{k}{d-k}\right) \sum_{i\in \mathcal{I}_d\backslash \{i_1,\ldots, i_{k}\}}  |x_i|^2.
\end{equation*}
Hence, 
\begin{align*}
\Exp{\big\|\mathcal{C}(x)- \mathbb{E}[\mathcal{C}(x)]\big\|^2} &\leq
\frac{k'(d-k-k')}{(d-k)^2}\frac{d-k}{d}\|x\|^2 = \frac{k'(d-k-k')}{(d-k)d}\|x\|^2.
\end{align*}
Then, let us calculate $\eta$.
\begin{align*}
\big\| \mathbb{E}[\mathcal{C}(x)]-x\big\|^2 &=\sum_{i\in \mathcal{I}_d\backslash \{i_1,\ldots, i_{k}\}} \left(\frac{d-k-k'}{d-k}\right)^2|x_i|^2\\
&\leq  \frac{(d-k-k')^2}{(d-k)d}\|x\|^2.
\end{align*}
Thus, $\eta =\frac{d-k-k'}{\sqrt{(d-k)d}}$.

\section{Proof of Proposition~\ref{prop2}}\label{secproofp5}

We first calculate $\omega$. Let $x\in\mathbb{R}^d$.
\begin{align*}
\big\|\mathcal{C}(x)- \mathbb{E}[\mathcal{C}(x)]\big\|^2&=\sum_{j\in\{j_1,\ldots,j_{k}\}}  \left(\frac{k'-k}{k}\right)^2|x_{i_j}|^2+
\sum_{i\in  \{i_1,\ldots, i_{k'}\}\backslash \{i_{j_1},\ldots, i_{j_{k}}\}} |x_i|^2
\end{align*}
Therefore, by taking the expectation over the random indexes $i_{j_1},\ldots,i_{j_{k}}$,
\begin{align*}
\Exp{\big\|\mathcal{C}(x)- \mathbb{E}[\mathcal{C}(x)]\big\|^2} &=\sum_{j=1}^{k'}
\left( \frac{k}{k'}\left(\frac{k'-k}{k}\right)^2 +\frac{k'-k}{k'}\right)|x_{i_j}|^2\\
&=\frac{k'-k}{k}\sum_{j=1}^{k'}
|x_{i_j}|^2\\
&\leq\frac{k'-k}{k}\|x\|^2
\end{align*}
Then, let us calculate $\eta$:
\begin{align*}
\big\| \mathbb{E}[\mathcal{C}(x)]-x\big\|^2 =\sum_{i\in \mathcal{I}_d\backslash \{i_1,\ldots, i_{k'}\}} |x_i|^2 \leq  \frac{d-k'}{d}\|x\|^2.
\end{align*}

\section{Proof of Theorem~\ref{theo1}}

We have the descent property \citep[Lemma 4]{ric21}, for every $t\geq 0$,
\begin{align}
f(x^{t+1}) -f^\star &\leq f(x^t)  -f^\star -\frac{\gamma}{2} \sqnorm{\nabla f(x^t)} +\frac{ \gamma }{2}\sqnorm{g^{t+1}-\nabla f(x^t)}\notag\\
&\quad+ \left(\frac{ L}{2}-\frac{1}{2\gamma}\right)\sqnorm{x^{t+1}-x^t}\label{eqgergg}\\
&\leq (1-\gamma\mu) \big(f(x^t)  -f^\star\big)  +\frac{ \gamma }{2}\sqnorm{g^{t+1}-\nabla f(x^t)}+ \left(\frac{ L}{2}-\frac{1}{2\gamma}\right)\sqnorm{x^{t+1}-x^t}.\notag
\end{align}
Then, for every $t\geq 0$, conditionally on $x^t$, $h^t$ and $(h_i^t)_{i=1}^n$,
\begin{align*}
\Exp{\sqnorm{g^{t+1}-\nabla f(x^t)}} &=\Exp{\sqnorm{\frac{1}{n}\sum_{i=1}^n \Big(h_i^{t}-\nabla f_i(x^t) +\nu \mathcal{C}_i^t\big(\nabla f_i(x^t)-h_i^t\big) \Big) }}\\
&=\sqnorm{\frac{1}{n}\sum_{i=1}^n \Big(h_i^{t}-\nabla f_i(x^t) +\nu \Exp{\mathcal{C}_i^t\big(\nabla f_i(x^t)-h_i^t\big)} \Big)}\\
&\quad+\nu^2\Exp{\sqnorm{\frac{1}{n}\sum_{i=1}^n \Big( \mathcal{C}_i^t\big(\nabla f_i(x^t)-h_i^t\big)-\Exp{ \mathcal{C}_i^t\big(\nabla f_i(x^t)-h_i^t\big) } \Big) }}\\
&\leq \sqnorm{\frac{1}{n}\sum_{i=1}^n \Big(h_i^{t}-\nabla f_i(x^t) +\nu \Exp{\mathcal{C}_i^t\big(\nabla f_i(x^t)-h_i^t\big)} \Big)}\\
&\quad+\nu^2 \frac{\oma}{n}\sum_{i=1}^n \sqnorm{\nabla f_i(x^t)-h_i^t },
\end{align*}
where the last inequality follows from \eqref{eqbo}. In addition,
\begin{align*}
&\left\|\frac{1}{n}\sum_{i=1}^n \Big(h_i^{t}-\nabla f_i(x^t) +\nu \Exp{\mathcal{C}_i^t\big(\nabla f_i(x^t)-h_i^t\big)} \Big)\right\|\\
&\quad \leq \left\|\frac{1}{n}\sum_{i=1}^n \Big(\nu\big(h_i^{t}-\nabla f_i(x^t)\big) +\nu \Exp{\mathcal{C}_i^t\big(\nabla f_i(x^t)-h_i^t\big)} \Big)\right\|\\
&\quad\quad + (1-\nu)\left\|\frac{1}{n}\sum_{i=1}^n \big(h_i^{t}-\nabla f_i(x^t)\big)\right\|\\
&\quad \leq \frac{\nu}{n}\sum_{i=1}^n\left\|h_i^{t}-\nabla f_i(x^t)+ \Exp{\mathcal{C}_i^t\big(\nabla f_i(x^t)-h_i^t\big)}\right\|\\
&\quad\quad + \frac{1-\nu}{n}\sum_{i=1}^n \left\|h_i^{t}-\nabla f_i(x^t)\right\|\\
&\quad \leq  \frac{\nu\eta}{n}\sum_{i=1}^n \left\|\nabla f_i(x^t)-h_i^{t}\right\|+ \frac{1-\nu}{n}\sum_{i=1}^n \left\|\nabla f_i(x^t)-h_i^{t}\right\|\\
& \quad = \frac{1-\nu+\nu\eta}{n}\sum_{i=1}^n \left\|\nabla f_i(x^t)-h_i^{t}\right\|.
\end{align*}
Therefore, 
\begin{align*}
\sqnorm{\frac{1}{n}\sum_{i=1}^n \Big(h_i^{t}-\nabla f_i(x^t) +\nu \Exp{\mathcal{C}_i^t\big(\nabla f_i(x^t)-h_i^t\big)} \Big)} \leq \frac{(1-\nu+\nu\eta)^2}{n}\sum_{i=1}^n \sqnorm{\nabla f_i(x^t)-h_i^{t}},
\end{align*}
and, 
conditionally on $x^t$, $h^t$ and $(h_i^t)_{i=1}^n$,
\begin{align*}
\Exp{\sqnorm{g^{t+1}-\nabla f(x^t)}} &\leq 
\left((1-\nu+\nu\eta)^2+\nu^2\oma\right)\frac{1}{n}\sum_{i=1}^n \sqnorm{\nabla f_i(x^t)-h_i^{t}}.
\end{align*}

Thus, for every $t\geq 0$, conditionally on $x^t$, $h^t$ and $(h_i^t)_{i=1}^n$,
 \begin{align*}
\Exp{f(x^{t+1}) -f^\star} 
&\leq (1-\gamma\mu) \big(f(x^t)  -f^\star \big)  +\frac{ \gamma }{2}\big((1-\nu+\nu\eta)^2+\nu^2\oma\big)\frac{1}{n}\sum_{i=1}^n \sqnorm{\nabla f_i(x^t)-h_i^{t}}\\
&\quad+ \left(\frac{ L}{2}-\frac{1}{2\gamma}\right)\Exp{\sqnorm{x^{t+1}-x^t}}.
\end{align*}

Now, let us study the control variates $h_i^t$. Let $s>0$. Using the Peter--Paul inequality $\|a+b\|^2 \leq (1+s) \|a\|^2 + (1+s^{-1}) \|b\|^2$, for any vectors $a$ and $b$, 
we have, for every $t\geq 0$ and $i\in\mathcal{I}_n$,
\begin{align*}
\sqnorm{\nabla f_i(x^{t+1})-h_i^{t+1}}&=\sqnorm{h_i^{t}-\nabla f_i(x^{t+1}) +\lambda \mathcal{C}_i^t\big(\nabla f_i(x^t)-h_i^t\big)  }\\
&\leq (1+s)\sqnorm{h_i^{t}-\nabla f_i(x^{t}) +\lambda \mathcal{C}_i^t\big(\nabla f_i(x^t)-h_i^t\big)  }\\ 
&\quad+(1+s^{-1})\sqnorm{\nabla f_i(x^{t+1})-\nabla f_i(x^{t})}\\
&\leq (1+s)\sqnorm{h_i^{t}-\nabla f_i(x^{t}) +\lambda \mathcal{C}_i^t\big(\nabla f_i(x^t)-h_i^t\big)  } \\
&\quad+(1+s^{-1})L_i^2\sqnorm{x^{t+1}-x^{t}}.
\end{align*}
Moreover,  conditionally on $x^t$, $h^t$ and $(h_i^t)_{i=1}^n$, 
\begin{align*}
\Exp{\sqnorm{h_i^{t}-\nabla f_i(x^{t}) +\lambda \mathcal{C}_i^t\big(\nabla f_i(x^t)-h_i^t\big)  }}&=\sqnorm{h_i^{t}-\nabla f_i(x^{t}) +\lambda \Exp{\mathcal{C}_i^t\big(\nabla f_i(x^t)-h_i^t\big)} }\\
&\quad+\lambda^2\Exp{\sqnorm{ \mathcal{C}_i^t\big(\nabla f_i(x^t)-h_i^t\big)-\Exp{ \mathcal{C}_i^t\big(\nabla f_i(x^t)-h_i^t\big) }  }}\\
&\leq\sqnorm{h_i^{t}-\nabla f_i(x^{t}) +\lambda \Exp{\mathcal{C}_i^t\big(\nabla f_i(x^t)-h_i^t\big)} }\\
&\quad+\lambda^2\omega \sqnorm{\nabla f_i(x^t)-h_i^t}.
\end{align*}
In addition,
 \begin{align*}
\left\|h_i^{t}-\nabla f_i(x^{t}) +\lambda \Exp{\mathcal{C}_i^t\big(\nabla f_i(x^t)-h_i^t\big)}\right\|& \leq \left\|\lambda\big(h_i^{t}-\nabla f_i(x^{t})\big) +\lambda \Exp{\mathcal{C}_i^t\big(\nabla f_i(x^t)-h_i^t\big)}\right\|\notag\\
&\quad + (1-\lambda)\left\|h_i^{t}-\nabla f_i(x^t)\right\|\\
&\leq  \lambda\eta\left\|\nabla f_i(x^t)-h_i^{t}\right\|+ (1-\lambda)\left\|\nabla f_i(x^t)-h_i^{t}\right\|\notag\\
&=(1-\lambda+\lambda\eta)\left\|\nabla f_i(x^t)-h_i^t\right\|.
\end{align*}
Therefore, conditionally on $x^t$, $h^t$ and $(h_i^t)_{i=1}^n$, 
\begin{align*}
\Exp{\sqnorm{h_i^{t}-\nabla f_i(x^{t}) +\lambda \mathcal{C}_i^t\big(\nabla f_i(x^t)-h_i^t\big)  }}\leq\big((1-\lambda+\lambda\eta)^2+\lambda^2\omega\big) \sqnorm{\nabla f_i(x^t)-h_i^t}
\end{align*}
and
\begin{align*}
\Exp{\sqnorm{\nabla f_i(x^{t+1})-h_i^{t+1}}}&\leq (1+s)\big((1-\lambda+\lambda\eta)^2+\lambda^2\omega\big) \sqnorm{\nabla f_i(x^{t})-h_i^{t}}\\
&\quad+(1+s^{-1})L_i^2\Exp
{\sqnorm{x^{t+1}-x^{t}}},
\end{align*}
so that
\begin{align*}
\Exp{\frac{1}{n}\sum_{i=1}^n \sqnorm{\nabla f_i(x^{t+1})-h_i^{t+1}}}&
\leq (1+s)\big((1-\lambda+\lambda\eta)^2+\lambda^2\omega\big) \frac{1}{n}\sum_{i=1}^n \sqnorm{\nabla f_i(x^t)-h_i^{t}}\\
&\quad+(1+s^{-1})\tilde{L}^2\Exp
{\sqnorm{x^{t+1}-x^{t}}}.
\end{align*}

Let $\theta>0$; its value will be set to $\theta^\star$ later on. We introduce the Lyapunov function, for every $t\geq 0$,
\begin{equation*}
\Psi^t \eqdef f(x^t)-f^\star + \frac{\gamma}{2\theta}  \frac{1}{n}\sum_{i=1}^n \sqnorm{\nabla f_i(x^t)-h_i^{t}}.
\end{equation*}
Hence, for every $t\geq 0$, conditionally on $x^t$, $h^t$ and $(h_i^t)_{i=1}^n$, we have
\begin{align}
\Exp{\Psi^{t+1}} &\leq (1-\gamma\mu) \big(f(x^t)  -f^\star \big)\notag \\
&\quad +\frac{ \gamma }{2\theta}\Big( \theta\big((1-\nu+\nu\eta)^2+\nu^2\oma\big)\notag\\
&\quad+(1+s)\big((1-\lambda+\lambda\eta)^2+\lambda^2\omega\big)
\Big) \frac{1}{n}\sum_{i=1}^n \sqnorm{\nabla f_i(x^t)-h_i^{t}}\label{eqq1}\\
&\quad+ \left(\frac{ L}{2}-\frac{1}{2\gamma}+\frac{\gamma}{2\theta}(1+s^{-1})\tilde{L}^2\right)\!\Exp{\sqnorm{x^{t+1}-x^t}}.\notag
\end{align}
Making use of $r$ and $r_{\mathrm{av}}$ and setting 
$\theta = s(1+s)\frac{r}{r_{\mathrm{av}}}$, 
we can rewrite \eqref{eqq1} as:
\begin{align*}
\Exp{\Psi^{t+1}} &\leq (1-\gamma\mu) \big(f(x^t)  -f^\star \big) +\frac{ \gamma }{2\theta}\Big( \theta r_{\mathrm{av}}
+(1+s)r
\Big) \frac{1}{n}\sum_{i=1}^n \sqnorm{\nabla f_i(x^t)-h_i^{t}}\\
&\quad+ \left(\frac{ L}{2}-\frac{1}{2\gamma}+\frac{\gamma}{2\theta}(1+s^{-1})\tilde{L}^2\right)\!\Exp{\sqnorm{x^{t+1}-x^t}}\\
&=(1-\gamma\mu) \big(f(x^t)  -f^\star \big) +\frac{ \gamma }{2\theta} (1+s)^2 
 \frac{r}{n}\sum_{i=1}^n \sqnorm{\nabla f_i(x^t)-h_i^{t}}\\
&\quad+ \left(\frac{ L}{2}-\frac{1}{2\gamma}+\frac{\gamma}{2s^2}\frac{r_{\mathrm{av}}}{r}\tilde{L}^2\right)\!\Exp{\sqnorm{x^{t+1}-x^t}}.
\end{align*}
We now choose $\gamma$ small enough so that 
 \begin{equation}
L-\frac{1}{\gamma}+\frac{\gamma}{s^2}\frac{r_{\mathrm{av}}}{r}\tilde{L}^2 \leq 0.\label{eqgreg}
\end{equation}
A sufficient condition for \eqref{eqgreg} to hold is \citep[Lemma 5]{ric21}:
\begin{equation}
0<\gamma \leq \frac{1}{L+\tilde{L}\sqrt{\frac{r_{\mathrm{av}}}{r}}\frac{1}{s}}.\label{eqgamfek}
\end{equation}
Then, assuming that \eqref{eqgamfek} holds, we have, for every $t\geq 0$, conditionally on $x^t$, $h^t$ and $(h_i^t)_{i=1}^n$,
\begin{align*}
\Exp{\Psi^{t+1}} &\leq (1-\gamma\mu) \big(f(x^t)  -f^\star \big) +\frac{ \gamma }{2\theta} (1+s)^2 
 \frac{r}{n}\sum_{i=1}^n \sqnorm{\nabla f_i(x^t)-h_i^{t}}\\
&\leq \max\big(1-\gamma\mu,(1+s)^2 r\big)  \Psi^t.
\end{align*}

We see that $s$ must be small enough so that $(1+s)^2 r <1$; this is the case with 
$s=s^\star$, so that $(1+s^\star)^2 r = \frac{r+1}{2}<1$. 
Therefore, we set $s=s^\star$, and, accordingly, $\theta=\theta^\star$. Then, for every $t\geq 0$,
 conditionally on $x^t$, $h^t$ and $(h_i^t)_{i=1}^n$,
\begin{align*}
\Exp{\Psi^{t+1}} 
&\leq \max\big(1-\gamma\mu, {\frac{r+1}{2}}\big) \Psi^t.
\end{align*}
Unrolling the recursion using the tower rule yields \eqref{eqsdgerg}.

\section{Proof of Theorem~\ref{theo2}}

Using $L$-smoothness of $f$, we have, for every $t\geq 0$,
\begin{equation*}
f(x^{t+1})\leq f(x^t) + \langle \nabla f(x^t),x^{t+1}-x^t\rangle + \frac{L}{2}\|x^{t+1}-x^t\|^2.
\end{equation*}
Moreover, using convexity of $R$, we have, for every subgradient $u^{t+1}\in \partial R(x^{t+1})$,
\begin{equation}
R(x^t)\geq R(x^{t+1}) + \langle u^{t+1}, x^{t}-x^{t+1}\rangle.\label{khfjehgg}
\end{equation}
From the property that $\mathrm{prox}_{\gamma R}=(\mathrm{Id}+\gamma \partial R)^{-1}$~\citep{bau17}, it follows from 
$x^{t+1} = \mathrm{prox}_{\gamma R}(x^t - \gamma  g^{t+1})$ that
\begin{equation*}
0\in  \partial R(x^{t+1})+ \frac{1}{\gamma}(x^{t+1} - x^t+\gamma g^{t+1}).
\end{equation*}
So, we set $u^{t+1}\eqdef \frac{1}{\gamma}(x^{t} - x^{t+1})-g^{t+1}$. Using this subgradient in \eqref{khfjehgg} and replacing
$x^{t}-x^{t+1}$ by $\gamma(u^{t+1}+g^{t+1})$, we get, for every $t\geq 0$,
\begin{align*}
f(x^{t+1})+R(x^{t+1}) &\leq f(x^t) +R(x^t)  + \langle \nabla f(x^t)+u^{t+1},x^{t+1}-x^t\rangle + \frac{L}{2}\|x^{t+1}-x^t\|^2\\
&=f(x^t) +R(x^t)  - \gamma \langle \nabla f(x^t)+u^{t+1},g^{t+1}+u^{t+1}\rangle + \frac{L}{2}\gamma^2\|g^{t+1}+u^{t+1}\|^2\\
&=f(x^t) +R(x^t)  +\frac{ \gamma }{2}\|\nabla f(x^t)-g^{t+1}\|^2+ \left(\frac{\gamma^2 L}{2}-\frac{\gamma}{2}\right)\|g^{t+1}+u^{t+1}\|^2\\
&\quad -\frac{\gamma}{2} \|\nabla f(x^t)+u^{t+1}\|^2\\
&=f(x^t) +R(x^t)  +\frac{\gamma }{2}\|\nabla f(x^t)-g^{t+1}\|^2+ \left(\frac{ L}{2}-\frac{1}{2\gamma}\right)\|x^{t+1}-x^t\|^2\\
&\quad -\frac{\gamma}{2} \|\nabla f(x^t)+u^{t+1}\|^2
\end{align*}
Note that we recover \eqref{eqgergg} if $R=0$ and $u^t \equiv 0$.

Using the fact that for any vectors $a$ and $b$, $-\|a+b\|^2 \leq -\frac{1}{2} \|a\|^2 + \|b\|^2$, we have, for every $t\geq 0$,
\begin{align*}
-\frac{\gamma}{2} \|\nabla f(x^t)+u^{t+1}\|^2 &\leq -\frac{\gamma}{4} \|\nabla f(x^{t+1})+u^{t+1}\|^2 + \frac{\gamma}{2}  \|\nabla f(x^{t+1})-\nabla f(x^{t})\|^2\\
&\leq -\frac{\gamma}{4} \|\nabla f(x^{t+1})+u^{t+1}\|^2 + \frac{\gamma L^2}{2}  \|x^{t+1}-x^t\|^2.
\end{align*}
Hence, for every $t\geq 0$,
\begin{align*}
f(x^{t+1})+R(x^{t+1}) 
&\leq  f(x^t) +R(x^t)  +\frac{\gamma}{2}\|\nabla f(x^t)-g^{t+1}\|^2+ \left(\frac{ L}{2}-\frac{1}{2\gamma}+\frac{\gamma L^2}{2}\right)\|x^{t+1}-x^t\|^2\\
&\quad -\frac{\gamma}{4} \|\nabla f(x^{t+1})+u^{t+1}\|^2.
\end{align*}
It follows from the K{\L}  assumption \eqref{eqKL} that
\begin{align*}
f(x^{t+1})+R(x^{t+1}) -f^\star - R^\star
&\leq f(x^{t})+R(x^{t}) -f^\star - R^\star +\frac{\gamma}{2}\|\nabla f(x^t)-g^{t+1}\|^2\\
&\quad + \left(\frac{ L}{2}-\frac{1}{2\gamma}+\frac{\gamma L^2}{2}\right)\|x^{t+1}-x^t\|^2\\
&\quad-2\mu\frac{\gamma}{4}  \left(f(x^{t+1})+R(x^{t+1}) -f^\star - R^\star\right),
\end{align*}
so that
\begin{align*}
 \Big(1+\frac{\gamma\mu}{2}\Big)\left(f(x^{t+1})+R(x^{t+1}) -f^\star - R^\star\right)
&\leq f(x^{t})+R(x^{t}) -f^\star - R^\star +\frac{\gamma}{2}\|\nabla f(x^t)-g^{t+1}\|^2\\
&\quad+ \left(\frac{ L}{2}-\frac{1}{2\gamma}+\frac{\gamma L^2}{2}\right)\|x^{t+1}-x^t\|^2,
\end{align*}
and
\begin{align*}
f(x^{t+1})+R(x^{t+1}) -f^\star - R^\star
&\leq \Big(1+\frac{\gamma\mu}{2}\Big)^{-1}\big(f(x^{t})+R(x^{t}) -f^\star - R^\star\big) +\frac{\gamma}{2}\|\nabla f(x^t)-g^{t+1}\|^2\\
&\quad+ \left(\frac{ L}{2}-\frac{1}{2\gamma}+\frac{\gamma L^2}{2}\right)\|x^{t+1}-x^t\|^2.
\end{align*}
Let $s>0$. Like in the proof of Theorem~\ref{theo1}, we have
\begin{align*}
\Exp{\frac{1}{n}\sum_{i=1}^n \sqnorm{\nabla f_i(x^{t+1})-h_i^{t+1}}}&
\leq (1+s)\big((1-\lambda+\lambda\eta)^2+\lambda^2\omega\big) \frac{1}{n}\sum_{i=1}^n \sqnorm{\nabla f_i(x^t)-h_i^{t}}\\
&\quad+(1+s^{-1})\tilde{L}^2\Exp
{\sqnorm{x^{t+1}-x^{t}}}
\end{align*}
and
\begin{align*}
\Exp{\sqnorm{g^{t+1}-\nabla f(x^t)}} &\leq 
\left((1-\nu+\nu\eta)^2+\nu^2\oma\right)\frac{1}{n}\sum_{i=1}^n \sqnorm{\nabla f_i(x^t)-h_i^{t}}.
\end{align*}
We introduce the Lyapunov function, for every $t\geq 0$,
\begin{align*}
\Psi^t &\eqdef f(x^t)+R(x^t)-f^\star - R^\star + \frac{\gamma}{2\theta}  \frac{1}{n}\sum_{i=1}^n \sqnorm{\nabla f_i(x^t)-h_i^{t}},
\end{align*}
where 
$\theta = s(1+s)\frac{r}{r_{\mathrm{av}}}$.

Following the  same derivations as 
 in the proof of Theorem~\ref{theo1}, we obtain that, for every $t\geq 0$, conditionally on $x^t$, $h^t$ and $(h_i^t)_{i=1}^n$, 
\begin{align*}
\Exp{\Psi^{t+1}} &\leq \Big(1+\frac{\gamma\mu}{2}\Big)^{-1}\big(f(x^{t})+R(x^{t}) -f^\star - R^\star\big) \\
&\quad +\frac{ \gamma }{2\theta}\Big( \theta\big((1-\nu+\nu\eta)^2+\nu^2\oma\big)\\
&\quad+(1+s)\big((1-\lambda+\lambda\eta)^2+\lambda^2\omega\big)
\Big)\frac{1}{n}\sum_{i=1}^n \sqnorm{\nabla f_i(x^{t})-h_i^{t}}\\
&\quad+ \left(\frac{ L}{2}-\frac{1}{2\gamma}+\frac{\gamma L^2}{2}+\frac{\gamma}{2\theta}(1+s^{-1})\tilde{L}^2\right)\!\Exp{\sqnorm{x^{t+1}-x^t}}\\
&= \Big(1+\frac{\gamma\mu}{2}\Big)^{-1}\big(f(x^{t})+R(x^{t}) -f^\star - R^\star\big) \\
&\quad+\frac{ \gamma }{2\theta}\Big( \theta r_{\mathrm{av}}
+(1+s)r
\Big)\frac{1}{n}\sum_{i=1}^n \sqnorm{\nabla f_i(x^{t})-h_i^{t}}\\
&\quad+ \left(\frac{ L}{2}-\frac{1}{2\gamma}+\frac{\gamma L^2}{2}+\frac{\gamma}{2\theta}(1+s^{-1})\tilde{L}^2\right)\!\Exp{\sqnorm{x^{t+1}-x^t}}\\
&=\Big(1+\frac{\gamma\mu}{2}\Big)^{-1}\big(f(x^{t})+R(x^{t}) -f^\star - R^\star\big) +\frac{ \gamma }{2\theta} (1+s)^2 
\frac{r}{n}\sum_{i=1}^n \sqnorm{\nabla f_i(x^{t})-h_i^{t}}\\
&\quad+ \left(\frac{ L}{2}-\frac{1}{2\gamma}+\frac{\gamma L^2}{2}+\frac{\gamma}{2s^2}\frac{r_{\mathrm{av}}}{r}\tilde{L}^2\right)\!\Exp{\sqnorm{x^{t+1}-x^t}}.
\end{align*}
We now choose $\gamma$ small enough so that 
 \begin{equation*}
L-\frac{1}{\gamma}+\gamma L^2+\frac{\gamma}{s^2}\frac{r_{\mathrm{av}}}{r}\tilde{L}^2 \leq 0.
\end{equation*}
If we assume $\gamma\leq \frac{1}{L}$, a sufficient condition is 
 \begin{equation}
2L-\frac{1}{\gamma}+\frac{\gamma}{s^2}\frac{r_{\mathrm{av}}}{r}\tilde{L}^2 \leq 0.\label{eqgreg3}
\end{equation}
A sufficient condition for \eqref{eqgreg3} to hold is \citep[Lemma 5]{ric21}:
\begin{equation}
0<\gamma \leq \frac{1}{2L+\tilde{L}\sqrt{\frac{r_{\mathrm{av}}}{r}}\frac{1}{s}}.\label{eqgamfek2}
\end{equation}
Then, assuming that \eqref{eqgamfek2} holds, we have, for every $t\geq 0$, conditionally on $x^t$, $h^t$ and $(h_i^t)_{i=1}^n$,
\begin{align*}
\Exp{\Psi^{t+1}} &\leq \Big(1+\frac{\gamma\mu}{2}\Big)^{-1} \big(f(x^t)  +R(x^t)-f^\star-R^\star \big) +\frac{ \gamma }{2\theta} (1+s)^2 
 \frac{r}{n}\sum_{i=1}^n \sqnorm{\nabla f_i(x^t)-h_i^{t}}\\
&\leq \max\Big({\textstyle\frac{1}{1+\frac{1}{2}\gamma\mu}},(1+s)^2 r\Big)  \Psi^t.
\end{align*}
We set $s=s^\star$ and, accordingly, $\theta=\theta^\star$, so that  $(1+s^\star)^2 r = \frac{r+1}{2}<1$. Then, 
for every $t\geq 0$, conditionally on $x^t$, $h^t$ and $(h_i^t)_{i=1}^n$,
\begin{align*}
\Exp{\Psi^{t+1}} &\leq \max\left({\frac{1}{1+\frac{1}{2}\gamma\mu}},\frac{r+1}{2}\right)  \Psi^t.
\end{align*}
Unrolling the recursion using the tower rule yields \eqref{eqsdgerg2}.

\section{Proof of Theorem~\ref{thm:noncvx}}

    Let $\theta>0$; its value will be set to the prescribed value later on. We introduce the Lyapunov function, for every $t\geq 0$,
   \begin{equation*}
   \Psi^t \eqdef f(x^t)-f^{\inf} + \frac{\gamma}{2\theta}  \frac{1}{n}\sum_{i=1}^n \sqnorm{\nabla f_i(x^t)-h_i^{t}}.
   \end{equation*}
 According to \citep[Lemma 4]{ric21}, we have, for every $t\geq 0$,
   \begin{align}
   f(x^{t+1}) -f^{\inf} &\leq f(x^t)  -f^{\inf} -\frac{\gamma}{2} \sqnorm{\nabla f(x^t)} +\frac{ \gamma }{2}\sqnorm{g^{t+1}-\nabla f(x^t)} + \left(\frac{ L}{2}-\frac{1}{2\gamma}\right)\sqnorm{x^{t+1}-x^t}.\notag
   \end{align}
   As shown in the proof of Theorem~\ref{theo1}, we have, conditionally on $x^t$, $h^t$ and $(h_i^t)_{i=1}^n$,
\begin{align*}
\Exp{\sqnorm{g^{t+1}-\nabla f(x^t)}} &\leq 
\left((1-\nu+\nu\eta)^2+\nu^2\oma\right)\frac{1}{n}\sum_{i=1}^n \sqnorm{\nabla f_i(x^t)-h_i^{t}}.
\end{align*}
   As for the control variates $h_i^t$, as shown in the proof of Theorem \ref{theo1}, we have, conditionally on $x^t$, $h^t$ and $(h_i^t)_{i=1}^n$, 
   \begin{align*}
   \Exp{\frac{1}{n}\sum_{i=1}^n \sqnorm{\nabla f_i(x^{t+1})-h_i^{t+1}}}&
   \leq (1+s)\big((1-\lambda+\lambda\eta)^2+\lambda^2\omega\big) \frac{1}{n}\sum_{i=1}^n \sqnorm{\nabla f_i(x^t)-h_i^{t}}\\
   &\quad+(1+s^{-1})\tilde{L}^2\Exp
   {\sqnorm{x^{t+1}-x^{t}}}.
   \end{align*}

    Hence, for every $t\geq 0$, conditionally on $x^t$, $h^t$ and $(h_i^t)_{i=1}^n$, we have
   \begin{align}
   \Exp{\Psi^{t+1}} &\leq f(x^t) - f^{\inf} - \frac{\gamma}{2} \sqnorm{\nabla f(x^t)}\notag \\
   &\quad +\frac{ \gamma }{2\theta}\Big( \theta\big((1-\nu+\nu\eta)^2+\nu^2\oma\big)\notag+(1+s)\big((1-\lambda+\lambda\eta)^2+\lambda^2\omega\big)
   \Big) \frac{1}{n}\sum_{i=1}^n \sqnorm{\nabla f_i(x^t)-h_i^{t}}\notag\\
   &\quad+ \left(\frac{ L}{2}-\frac{1}{2\gamma}+\frac{\gamma}{2\theta}(1+s^{-1})\tilde{L}^2\right)\!\Exp{\sqnorm{x^{t+1}-x^t}}.\label{eqq1z}
   \end{align}
   Let $r \eqdef (1 - \lambda + \lambda \eta)^2 + \lambda^2 \omega, r_{\mathrm{av}}\eqdef (1 - \nu + \nu\eta)^2 + \nu^2 \oma$. Set $\theta \eqdef s(1+s)\frac{r}{r_{\mathrm{av}}}$. We can rewrite \eqref{eqq1z} as:
   \begin{align*}
   \Exp{\Psi^{t+1}} &\leq f(x^t) - f^{\inf} - \frac{\gamma}{2} \sqnorm{\nabla f(x^t)} +\frac{ \gamma }{2\theta}\Big( \theta r_{\mathrm{av}}
   +(1+s)r
   \Big) \frac{1}{n}\sum_{i=1}^n \sqnorm{\nabla f_i(x^t)-h_i^{t}}\\
   &\quad+ \left(\frac{ L}{2}-\frac{1}{2\gamma}+\frac{\gamma}{2\theta}(1+s^{-1})\tilde{L}^2\right)\!\Exp{\sqnorm{x^{t+1}-x^t}}\\
   &=f(x^t) - f^{\inf} - \frac{\gamma}{2} \sqnorm{\nabla f(x^t)} +\frac{ \gamma }{2\theta} (1+s)^2 
   \frac{r}{n}\sum_{i=1}^n \sqnorm{\nabla f_i(x^t)-h_i^{t}}\\
   &\quad+ \left(\frac{ L}{2}-\frac{1}{2\gamma}+\frac{\gamma}{2s^2}\frac{r_{\mathrm{av}}}{r}\tilde{L}^2\right)\!\Exp{\sqnorm{x^{t+1}-x^t}}.
   \end{align*}
   We now choose $\gamma$ small enough so that 
   \begin{equation}
   L-\frac{1}{\gamma}+\frac{\gamma}{s^2}\frac{r_{\mathrm{av}}}{r}\tilde{L}^2 \leq 0.\label{eqgreg5}
   \end{equation}
   A sufficient condition for \eqref{eqgreg5} to hold is \citep[Lemma 5]{ric21}:
   \begin{equation}
   0<\gamma \leq \frac{1}{L+\tilde{L}\sqrt{\frac{r_{\mathrm{av}}}{r}}\frac{1}{s}}.\label{eqgamfek5}
   \end{equation}
   Then, assuming that \eqref{eqgamfek5} holds, we have, for every $t\geq 0$, conditionally on $x^t$, $h^t$ and $(h_i^t)_{i=1}^n$,
   \begin{align*}
   \Exp{\Psi^{t+1}} &\leq f(x^t) - f^{\inf} - \frac{\gamma}{2} \sqnorm{\nabla f(x^t)} +\frac{ \gamma }{2\theta} (1+s)^2 
   \frac{r}{n}\sum_{i=1}^n \sqnorm{\nabla f_i(x^t)-h_i^{t}}.
   \end{align*}
 We have chosen $s$ so that $(1+s)^2 r = 1$. 
 Hence, using the tower rule, we have, for every $t\geq 0$,
   \begin{align*}
      \Exp{\Psi^{t+1}} \leq \Exp{\Psi^t} - \frac{\gamma}{2}\Exp{\sqnorm{\nabla f(x^t)}}.
   \end{align*}
   Let $T\geq 1$. By summing up the inequalities for $t=0, \cdots, T-1$, we get 
   \begin{align*}
      0 \leq \Exp{\Psi^T} \leq \Psi^0 - \frac{\gamma}{2} \sum_{t=0}^{T-1} \Exp{\sqnorm{\nabla f(x^t)}}.
   \end{align*}
   Multiplying both sides by $\frac{2}{\gamma T}$ and rearranging the terms, we get
   \begin{align*}
      \frac{1}{T}\sum_{t=0}^{T-1} \Exp{\sqnorm{\nabla f(x^t)}} \leq \frac{2}{\gamma T} \Psi^0,
   \end{align*}
   where the left hand side can be interpreted as $\mathbb{E}\left[\left\|\nabla f(\hat{x}^{T})\right\|^{2}\right]$, where $\hat{x}^{T}$ is chosen from $x^{0}, x^{1}, \ldots, x^{T-1}$ uniformly at random.

\end{document}